\newif\ifonecolumn 

\documentclass[journal,twoside,web]{ieeecolor}\onecolumnfalse
\usepackage{generic}
\usepackage{amsmath,amssymb,amsfonts}
\usepackage{algorithmic}
\usepackage{graphicx}
\usepackage{textcomp}

\pdfoutput=1
\usepackage{bm}
\usepackage{mathabx}
\usepackage[ruled,vlined,titlenotnumbered]{algorithm2e} 
\usepackage{graphicx} 
\usepackage{subcaption}
\usepackage{epsfig} 
\usepackage{cancel}
\usepackage{amssymb}
\usepackage{resizegather}

\usepackage{todonotes}

\usepackage{ifthen}

\newboolean{include-notes}
\setboolean{include-notes}{false}
\newcommand{\jfnote}[1]{\ifthenelse{\boolean{include-notes}}%
    {\textcolor{blue}{\textbf{[JF: #1]}}}{}}

\newboolean{show-new}
\setboolean{show-new}{true}
\newboolean{mark-new}
\setboolean{mark-new}{false}
\newcommand{\new}[1]{\ifthenelse{\boolean{show-new}}{%
    \ifthenelse{\boolean{mark-new}}%
        {\textcolor{blue}{#1}}
        {#1}
    }{}
        }

\usepackage{enumerate}

\usepackage{url}
\usepackage{hyperref}

\usepackage[%
    style=numeric-comp,
    sorting=nyt,
    backend=biber,
    sortcites=true,
    doi=false,
    firstinits=true,
    hyperref,
    isbn=false,
    eprint=false,
    maxcitenames=20,
    maxbibnames=20,
    block=none]
    {biblatex}
    
\DeclareSourcemap{
    \maps[datatype = bibtex]{
        \map{
            \step[fieldsource=journal, match=\regexp{\([[:upper:]]*[[:punct:]]*[[:digit:]]*\)}, replace=]
        }
    }      
}
      
\renewbibmacro{in:}{}
\AtEveryBibitem{
    \clearlist{language}
    \clearfield{pages}
    \clearfield{address}
    \clearfield{location}
    \clearfield{month}
}

\DeclareBibliographyCategory{needsurl}
\newcommand{\entryneedsurl}[1]{\addtocategory{needsurl}{#1}}
\renewbibmacro*{url+urldate}{%
    \ifcategory{needsurl}{
        \printfield{url}%
        \iffieldundef{urlyear}
          {}
          {\setunit*{\addspace}%
          \printurldate}
    }
    {}
}

\entryneedsurl{Mitchell2004}

\newcommand{\B}{\mathcal{B}}
\newcommand{\C}{\mathcal{C}}
\newcommand{\D}{\mathcal{D}}

\newcommand{\K}{\mathcal{K}}
\renewcommand{\L}{\mathcal{L}}
\newcommand{\M}{\mathcal{M}}
\newcommand{\N}{\mathcal{N}}

\newcommand{\Q}{\mathcal{Q}}

\newcommand{\U}{\mathcal{U}}
\newcommand{\V}{\mathcal{V}}

\newcommand{\DD}{\mathfrak{D}}

\newcommand{\RR}{\mathbb{R}}

\newcommand{\UU}{\mathfrak{U}}

\newcommand{\XX}{\mathbb{X}}

\ifonecolumn
\newcommand{\fixwidth}[1]{#1} 
\else
\newcommand{\fixwidth}[1]{\resizebox{\columnwidth}{!}{\ensuremath{\displaystyle{#1}}}} 
\fi

\newcommand{\Dx}{\hat\D(x)}

\newcommand{\bu}{\bm{u}}
\newcommand{\bdelta}{\bm{d}}

\newcommand{\bbeta}{\bm{\beta}}

\newcommand{\bx}{\xi}

\newcommand{\Disc}{\text{Disc}}
\DeclareMathOperator{\interior}{int}

\newtheorem{definition}{Definition}

\newtheorem{corollary}{Corollary}
\newtheorem{remark}{Remark}
\newtheorem{proposition}{Proposition}
\newtheorem{assumption}{Assumption}

\title{
A General Safety Framework for Learning-Based Control in Uncertain Robotic Systems
}
\author{
Jaime F. Fisac\textsuperscript{$*$1}, \and Anayo K. Akametalu\textsuperscript{$*$1}, \and Melanie N. Zeilinger\textsuperscript{2},\\ \and Shahab Kaynama\textsuperscript{3}, \and Jeremy Gillula\textsuperscript{4}, and \and Claire J. Tomlin\textsuperscript{1}
\thanks{
    \textsuperscript{$*$} The first two authors contributed equally. \newline\indent
    This work was supported by the NSF CPS project ActionWebs under grant 0931843, NSF CPS project FORCES under grant 1239166, and by ONR under the HUNT, SMARTS and Embedded Humans MURIs, and by AFOSR under the CHASE MURI. The research of J. F. Fisac received funding from the ``la Caixa" Foundation. The research of A.K. Akametalu received funding from the NSF Bridge to Doctorate program. The research of M.N. Zeilinger received funding from the EU FP7 (FP7/2007-2013) under grant PIOFGA-2011-301436-ÒCOGENTÓ.}
\thanks{
    \hspace{-1.5em}\
    \textsuperscript{1} Department of Electrical Engineering and Computer Sciences, 
        University of California, Berkeley.
        Cory Hall, Berkeley, CA 94720, United States.\newline\indent
    \textsuperscript{2} Department of Mechanical and Process Engineering, ETH Zurich.
    R{\"a}mistrasse 101, 8092 Z{\"u}rich, Switzerland.\newline\indent
    \textsuperscript{3} Clearpath Robotics.
    1425 Strasburg Rd, Suite 2A, Kitchener, ON N2R 1H2, Canada.\newline\indent
    \textsuperscript{4} Electronic Frontier Foundation.
    815 Eddy St, San Francisco, CA 94109, United States.\newline\indent
    Email: {\tt\small \{jfisac, kakametalu, tomlin\}~@eecs.berkeley.edu, mzeilinger@ethz.ch, skaynama@clearpath.ai, jeremy@eff.org }}%
}

\begin{document}
\maketitle
\thispagestyle{empty}
\pagestyle{empty}

\begin{abstract}

The proven efficacy of learning-based control schemes strongly motivates their application to 
robotic systems operating in the physical world.
However, guaranteeing correct operation during the learning process is currently an unresolved issue, which is of vital importance in safety-critical systems.
We propose a general safety framework based on Hamilton-Jacobi reachability methods that can work in conjunction with an arbitrary learning algorithm.
The method exploits approximate knowledge of the system dynamics to guarantee constraint satisfaction while minimally interfering with the learning process.
We further introduce a Bayesian mechanism that refines the safety analysis as the system acquires new evidence, reducing initial conservativeness when appropriate while strengthening guarantees through real-time validation.
The result is a least-restrictive, safety-preserving control law that intervenes only when (a) the computed safety guarantees require it, or (b)~confidence in the computed guarantees decays in light of new observations.
We prove theoretical safety guarantees combining probabilistic and worst-case analysis and demonstrate the proposed framework experimentally on a quadrotor vehicle. Even though safety analysis is based on a simple point-mass model, the quadrotor successfully arrives at a suitable controller by policy-gradient reinforcement learning without ever crashing, and safely retracts away from a strong external disturbance introduced during flight.

\end{abstract}

\section{Introduction}\label{introduction}

\begin{figure}
\begin{center}
\includegraphics[
 width=.48\textwidth]{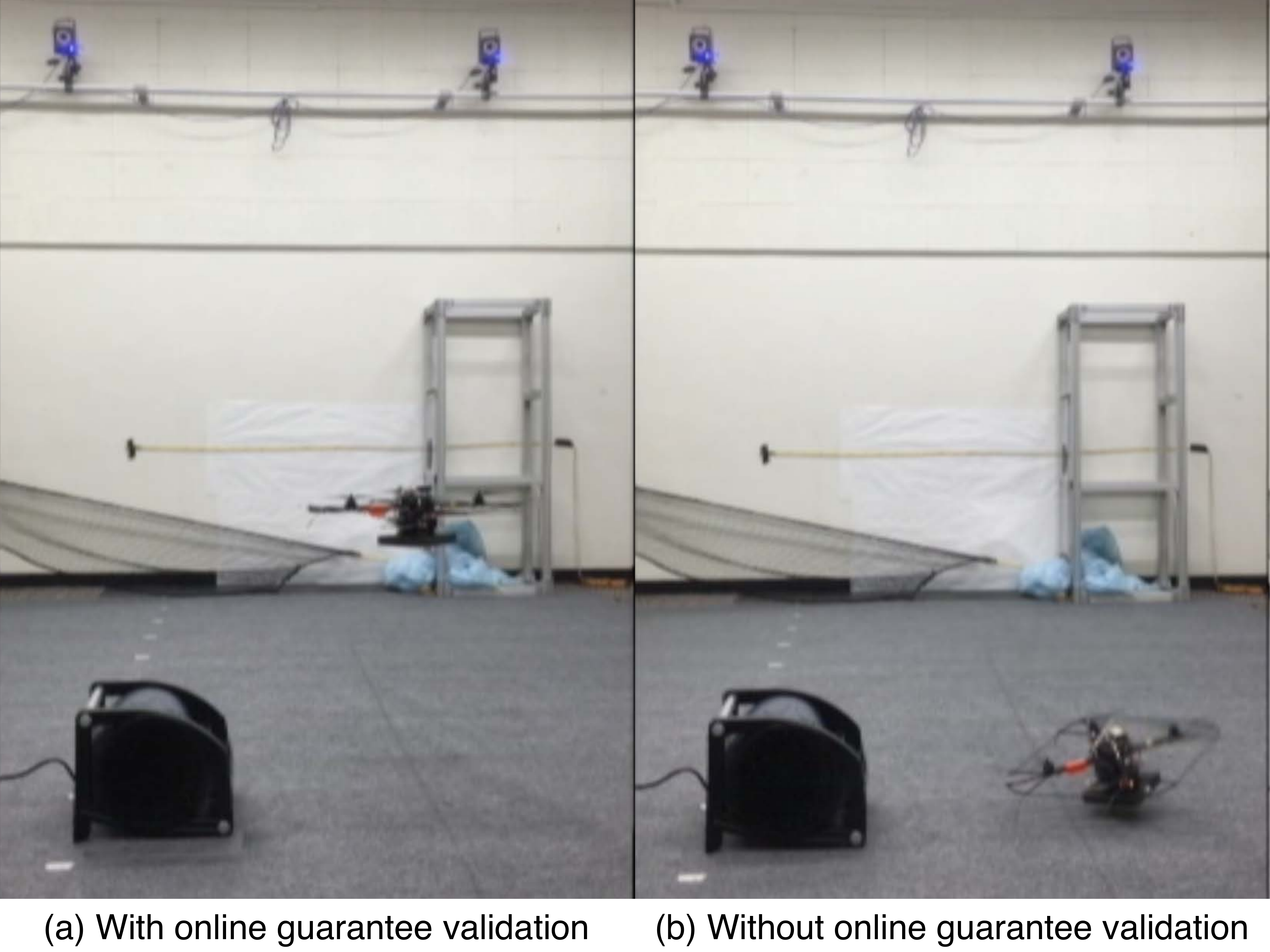}
\caption{Hummingbird quadrotor learning a vertical flight policy under the requirement of not colliding.
When the fan is turned on, the system experiences an unmodeled disturbance that it has not previously encountered.
This can lead to a ground collision even under robust safety policies
(right).
The proposed Bayesian validation method
detects the inconsistency and prevents the vehicle from entering the uncertain region~(left).
\newline Video:
{\tt\href{http://eecs.berkeley.edu/~jfisac/safelearning}{\nolinkurl{www.eecs.berkeley.edu/~jfisac/safelearning}}}
\vspace{-10pt} 
\label{fig:quad}}
\end{center}
\end{figure}

Learning-based methods in control and artificial intelligence are generating a considerable amount of excitement in the research community.
The auspicious results of deep reinforcement learning schemes in virtual environments such as arcade videogames \cite{Mnih2015a} and physics simulators \cite{Schulman2015}, make these techniques extremely attractive for robotics applications, in which complex dynamics and hard-to-model environments limit the effectiveness of purely model-based approaches.
However, the difficulty of interpreting the inner workings of many machine learning algorithms (notably in the case of deep neural networks), makes it challenging to make meaningful statements about the behavior of a system during the learning process,
especially while the system has not yet converged to a suitable control policy.
While this may not be a critical issue in a simulated reality, it can quickly become a limiting factor when attempting to put such an algorithm in control of a system in the physical world,
where certain failures, such as collisions, can result in damage that would severely hinder or even terminate the learning process, in addition to material loss or human injury. 
We refer to systems in which certain failure states are unacceptable as \emph{safety-critical}.

In the last decade, learning-based control schemes have been successfully demonstrated in robotics applications in which the safety-critical aspects were effectively removed or mitigated, typically by providing a manual fallback mechanism or retrofitting the environment to allow safe failure. 
In \cite{Abbeel2007},~\cite{Coates2008} a trained pilot was able to remotely take over control of the autonomous helicopter at any time; the power slide car maneuvers in \cite{Kolter2010} were performed on an empty test track; and the aerobatic quadrotor in \cite{Lupashin2010} was enclosed in a safety net. While mostly effective, these \emph{ad hoc} methods tend to come with their own issues
(pilot handoffs, for instance, are notoriously prone to result in accidents \cite{Hobbs2010})
and do not generalize well beyond the context of the particular demonstration.
It therefore seems imperative to develop principled and provably correct approaches to safety,
attuned to the exploration-intense needs of learning-based algorithms,
that can be built into the autonomous operation of learning robotic systems.

Current efforts in policy transfer learning propose training an initial control policy in simulation and then carrying it over to the physical system \cite{Christiano2016}.
While progress in this direction is likely to reduce overall training time,
it does not eliminate the risk of catastrophic system misbehavior.
State-of-the art neural network policies have been shown to be vulnerable to small changes between training and testing conditions \cite{Huang2017}, which inevitably arise between simulated and real systems. Guaranteeing correct behavior of simulation-trained schemes in the real world thus remains an important unsolved problem.

Providing guarantees about a system's evolution inevitably requires some form of knowledge about the causal mechanisms that govern it.
Fortunately, in practice it is never the case that the designer of a robotic system has no knowledge whatsoever of its dynamics:
making use of approximate knowledge
is both possible and, we argue, advantageous for safety.
Yet, perfect knowledge of the dynamics can hardly if ever be safely assumed either.
This
motivates searching for points of rapprochement between data-driven and model-based techniques.

We identify three key properties that we believe any general safe learning framework should satisfy:
\begin{itemize}
\item \textbf{\emph{High confidence}}. The framework should be able to keep the system safe with high probability given the available knowledge about the system and the environment.
\item \textbf{\emph{Modularity}}. The framework should work in conjunction with an arbitrary learning-based control algorithm, without requiring modifications to said algorithm.
\item \textbf{\emph{Minimal intervention}}. The framework should not interfere with the learning process unless deemed strictly necessary to ensure safety, and should return control to the learning algorithm as soon as possible.
\end{itemize}

We can use these criteria to evaluate the strengths and shortcomings of existing approaches to safety in intelligent systems, and place our work in the context of prior research.

\subsection*{Related Work}
Early proposals of safe learning date back to the turn of the century. Lyapunov-based reinforcement learning \cite{Perkins2003} allowed a learning agent to switch between a number of pre-computed ``base-level" controllers with desirable safety and performance properties; this enabled solid theoretical guarantees at the expense of substantially constraining the agent's behavior; in a similar spirit, later work has considered constraining policy search to the space of stabilizing controllers \cite{Roberts2011}.

In risk-sensitive reinforcement learning \cite{Geibel2005}, the expected return was heuristically weighted with the probability (risk) of reaching an ``error state"; while this allowed for more general learning strategies, no guarantees could be derived from the heuristic effort. Nonetheless, the ideal problem formulation proposed in the paper, to maximize performance subject to some maximum allowable risk, inspired later work (see \cite{Garcia2015a} for a survey) and is very much aligned with our own goals.

More recently, \cite{Moldovan2012} proposed an ergodicity-based safe exploration policy for Markov decision processes (MDPs) with uncertain transition measures, which imposed a constraint on the probability, under the current belief, of being able to return to the starting state.
While
practical online methods for updating the system's belief on the transition dynamics are not discussed,
and the toy grid-world demonstrations fall short of capturing
the criticality of dynamics in many real-world safety problems,
the probabilistic safety analysis is extremely powerful, and our work certainly takes inspiration from it.
\new{%
Recent safe exploration efforts in robotics concurrent with our work
use Gaussian processes to model uncertain dynamics, but
restrict safety analysis to local stability (i.e. region of attraction)
and do not consider state constraints
\cite{Berkenkamp2016}, \cite{Berkenkamp2017}.%
}%

The robust model-predictive control approach in \cite{Aswani2013} learns about system dynamics for performance only, while enforcing constraints based on a robust nominal model. The method was successfully demonstrated on problems with nontrivial dynamics, including quadrotor flight.
However,
using an \emph{a priori} model for safety at best
constrains the system's ability to explore,
and at worst may fail to keep the real system safe.

To explicitly account for model uncertainty, the safety problem can be studied as a differential game \cite{Mitchell2005}, in which the controller must keep the system within the specified state constraints
(i.e. away from failure states)
in spite of the actions of an adversarial disturbance:
the optimal solution to this reachability game, obtainable through Hamilton-Jacobi methods \cite{Mitchell2005a},~\cite{Osher2003}, has been used to guarantee safety in a variety of engineering problems \cite{Ding2011},~\cite{
Chen2015},~\cite{Herbert2017}.
This robust, worst-case analysis determines a safe region in the state space and a control policy to remain inside it;
a related approach involves ensuring invariance through ``barrier functions" \cite{Prajna2004a},~\cite{Sloth2012}.
A key advantage is that in the interior of this \emph{safe set} one can execute any desired action, as long as the safe control is applied at the boundary: in this sense, the technique yields a \emph{least-restrictive} control law,
which naturally lends itself to minimally constrained learning-based control. Initial work exploring this was presented in \cite{Gillula2012},~\cite{Gillula2012a}.

The above methods are subject, however, to the fundamental limitation of any model-based safety analysis, namely, the contingency of all guarantees upon the validity of the model.
This faces designers with a difficult tradeoff.
On the one hand, if, in order to guarantee safety under large or poorly understood uncertainty, they assume conservative bounds on model error, this will reduce the computed safe set, and thereby restrict the freedom of the learning algorithm. If, on the other hand, the assumed bounds fail to fully capture the true evolution of the state, the theoretical guarantees derived from the model may not in fact apply to the real system.

\subsection*{Contribution}

In this work we propose a novel general safety framework that combines model-based control-theoretical analysis with data-driven Bayesian inference to construct and maintain high-probability guarantees around an arbitrary learning-based control algorithm. Drawing on Hamilton-Jacobi robust optimal control techniques, it defines a \emph{least-restrictive} supervisory control law, which allows the system to freely execute its learning-based policy almost everywhere, but imposes a computed action at states where it is deemed critical for safety. The safety analysis is refined through Bayesian inference in light of newly gathered evidence, both avoiding excessive conservativeness and improving reliability by rapidly imposing the computed safe actions if confidence in model-based guarantees decreases due to unexpected observations.

This paper
consolidates the preliminary work presented in \cite{Gillula2012},~\cite{Gillula2012a},~\cite{Akametalu2014}, extending the theoretical results to provide a unified treatment of model learning and guarantee validation, and presenting significant novel experimental results.
To our knowledge this is the first work in the area of reachability analysis that
reasons online about the validity of computed guarantees and uses a
resilient mechanism to continue exploiting them under inaccurate prior assumptions on model error.

Our framework relies on reachability analysis for the model-based safety guarantees, and on Gaussian processes for the online Bayesian analysis.
It is important to acknowledge that both of these techniques are computationally intensive and scale poorly with the dimensionality of the underlying continuous spaces, which can generally limit their applicability to complex dynamical systems.
However, recent compositional approaches have dramatically increased the tractability of lightly coupled high-dimensional systems%
~\cite{Kaynama2013},%
~\cite{Kaynama2015},%
~\cite{Chen2015},%
~\cite{Chen2016a},%
~while new analytic solutions entirely overcome the ``curse of dimensionality" in some relevant cases \cite{Darbon2016},~\cite{Kirchner2017}.
The key contribution of this work is in the principled methodology for incorporating safety into learning-based systems:
we thus focus our examples on problems of low dimensionality, implicitly bypassing the computational issues,
and note that our method can readily be used in conjunction with these decomposition techniques to extend its application to more complex systems.

We demonstrate our method on a quadrotor vehicle learning to track a vertical trajectory close to the ground (Fig. \ref{fig:quad}), using a policy gradient algorithm \cite{Kolter2009}.
The reliability of our method is evidenced under uninformative policy initializations, inaccurate safe set estimation and strong unmodeled disturbances.

The remainder of the paper is organized as follows: In Section \ref{sec:formulation} we
introduce the modeling framework and
formally state the safe learning problem. Section \ref{sec:analysis} presents the differential game analysis and derives some important properties. The proposed methodology is described in Section \ref{sec:solution} with the proofs of its fundamental guarantees, as well as a computationally tractable alternative with weaker, but practically useful, properties. Lastly, in Section \ref{sec:results} we present the experimental results.

\section{Problem Formulation \label{sec:formulation}}
\subsection{System Model and State-Dependent Uncertainty\label{subsec:dynamics}}

The analysis in this paper considers a fully observable system whose underlying dynamics are assumed deterministic, but only \emph{partially} known. This modeling framework can in practice be applied to a wide range of
systems for which an approximate dynamic model is available but the exact behavior is hard to model \emph{a priori} (due to manufacturing tolerances, aerodynamic effects, uncertain environments, etc.).

We can formalize this as a dynamical system with state $x\in\RR^n$, and two inputs, $u\in\U\subset\RR^{n_u},  d\in\D\subset\RR^{n_d}$
(with $\U$ and $\D$ compact)
which we will refer to as the \emph{controller} and the \emph{disturbance}:
\begin{equation}\label{fxud}
\dot{x} = f(x,u, d).
\end{equation}
In this context, however, 
$d$ is thought of as a deterministic state-dependent disturbance capturing unmodeled dynamics, given by an unknown Lipschitz function $d:\RR^n\to\D$. 
That is, we could in principle write the unknown dynamics as
$
F(x,u) = f\big(x,u,d(x)\big)
$. Unlike $F$, $f$ is a known function, with all uncertainty captured by $d(\cdot)$. 
The flow field $f: \RR^n \times \U \times \D\rightarrow\RR^n$ is assumed uniformly continuous and bounded,
as well as Lipschitz in $x$ and $ d$ for all $u$:
this ensures
that the unknown dynamics $F$ are Lipschitz in $x$.

Letting $\UU $ and $\DD$ denote the collections of measurable%
    \footnote{A function $f:X\to Y$ between two measurable spaces $(X,\Sigma_X)$ and $(Y,\Sigma_Y)$
    is said to be measurable if the preimage of a measurable set in $Y$ is a measurable set in $X$, that is:
    $\forall V\in\Sigma_Y, f^{-1}(V)\in\Sigma_X$, with $\Sigma_X,\Sigma_Y$ $\sigma$-algebras on $X$,$Y$.}
functions $\bm u: [0,\infty)\to \U $ and $\bm d: [0,\infty)\to \D$ respectively,
and allowing the controller and disturbance to choose any such signals,
the evolution of the system
from any initial state $x$
is determined (see for example \cite{Coddington1955}, Ch. 2, Theorems 1.1, 2.1) by the unique continuous trajectory $\bx:[0,\infty)\to\RR^n$ solving
\begin{equation}\label{eq:xdot}
\begin{split}
\dot{\bx}(s) &= f(\bx(s),\bu(s),\bdelta(s)), \text{ a.e. }s\ge 0,\\
\bx(0) &= x.
\end{split}
\end{equation}
Note that this is a solution in Carath\'eodory's \emph{extended sense}, that is, it satisfies the differential equation \emph{almost everywhere} (i.e. except on a subset of Lebesgue measure zero).

Since $d(x)$ is unknown, we attempt to bound it at each state by a compact set $\hat{\D}(x)\subseteq\D$, which is allowed to vary in the state space.
\new{
    In Section~\ref{sec:analysis}, we present a robust, least-restrictive safety control law that
    enforces
    constraint satisfaction subject to $d(x) \in \hat{\D}(x)$.
    In Section~\ref{sec:solution}, we present a Bayesian approach to find a bound $\hat{\D}(x)$ based on a Gaussian process model of $d(x)$.
    Our overall approach therefore combines robust (worst-case) guarantees with Bayesian (probabilistic) analysis,
    by constructing the disturbance bound to reflect the local uncertainty around the inferred disturbance function.%
    \footnote{
        \new{
            Alternative methods to providing the disturbance bounds 
            (for example, a conservative \emph{a priori} estimate, or a system identification procedure)
            are possible, provided some basic conditions to ensure that the dynamical system resulting from \eqref{eq:xdot} with the restriction $ d\in\Dx$ remains well defined.
            For the interested reader, sufficient conditions on $\Dx$ are discussed in the Appendix.
        }
    }
}

Any model-based safety guarantees for the system will require that the bound $\Dx$ correctly captures the unknown part of the dynamics given by $d(x)$, at least at some critical set of states $x$ (discussed in Section~\ref{sec:analysis}). One key insight in this work is that the system should take action to ensure safety not only when the model predicts that this action may be necessary, but also when the system detects that the model itself may become unreliable in the predictable future.

We state here a preliminary result that will be useful later on in the paper,
and introduce the notion of local model reliability.

\begin{proposition}\label{prop:locally_correct}
If $d(x)\in\interior\Dx$ and the set-valued map $\hat\D:\RR^n\to2^{\D}$ is Lipschitz-continuous under the Hausdorff metric%
\footnote{The Hausdorff metric (or Hausdorff distance) between any two sets $A$ and $B$ in a metric space $(M,d_M)$ is defined as
${d_{H}(A,B) = \max\{\,\sup_{a \in A} \inf_{b \in B} d_M(a,b),\, \sup_{b \in B} \inf_{a \in A} d_M(a,b)\,\}}$.}%
, then there exists $ \Delta t>0$ such that all possible trajectories followed by the system starting at $x$ will satisfy $d\big(\xi(\tau)\big)\in\hat\D\big(\xi(\tau)\big)$ for all $\tau\in[t,t+ \Delta t]$.
\begin{proof}
Let $L_{\hat\D}$ be the Lipschitz (Hausdorff) constant of $\hat\D$, $L_d$ the Lipschitz constant of $d$, and $C_f$ a norm bound on the dynamics $f$.
We then have that over an arbitrary time interval $[t,t+ \Delta t]$,
regardless of the control and disturbance signals $\bm{u}(\cdot)$, $\bm{ d}(\cdot)$,
any system trajectory starting at ${\xi(t)=x}$ satisfies
${|\xi(\tau)-x|\le C_f  \Delta t}, {\forall \tau\in[t,t+ \Delta t]}$.
This implies both ${|d(\xi(\tau))-d(x)|\le L_d C_f  \Delta t}$
and ${d_H\big(\hat \D(\xi(\tau)),\hat\D(x)\big)\le L_{\hat\D} \C_f  \Delta t}$.
Requiring that the open ball ${B\big(d(x),(L_d+L_{\hat\D}) C_f  \Delta t\big)}$ be contained in $\Dx$
ensures ${d(\xi(\tau))\in\hat\D(\xi(\tau))}$. Since $d(x)\in\interior\Dx$, there must exist a small enough $ \Delta t>0$ for which this condition is met.
\end{proof}
\end{proposition}

\new{
We can further quantify this $\Delta t$ through the signed distance%
\footnote{
    For any nonempty set $\M\subset\RR^m$, the \emph{signed distance function} ${s_{\M}:\RR^m\to\RR}$ is defined as $\inf_{y\in\M} |z-y|$ for points $z\in\RR^m\setminus\M$ and $-\inf_{y\in\RR^m\setminus\M} |z-y|$ for points $z\in\M$, where $|\cdot|$ denotes a norm on $\RR^m$.
}
to $\Dx$ at the current $d(x)$.
}
\begin{corollary}
If the Lipschitz constants are known, then ${d(x)\in\interior\Dx}$ implies ${d(\xi(\tau))\in\hat\D(\xi(\tau))}$ for all times ${\tau\in[t,t+ \Delta t]}$, with
\[ \Delta t=\frac{-s_{\Dx}\big(d(x)\big)}{(L_d+L_{\hat\D}) C_f}.\]
\end{corollary}

The disturbance bounds $\hat\D$ derived in this paper satisfy the hypothesis of Proposition \ref{prop:locally_correct} (see Appendix for details), and we thus refer to the condition $d(x)\in\interior\Dx$ as the model being \emph{locally reliable} at $x$.

Finally, we assume that the effect of the disturbance on the dynamics is independent of the action applied by the controller.
\begin{equation}\label{eq:decoupled}
\dot{x} = f\big(x,u,d(x)\big) = g\big(x,u\big) + h\big(d(x)\big).
\end{equation}
with $g:\RR^n\times\RR^{n_u}\to\RR^n, h:\RR^{n_p}\to\RR^n$, where $g$, and $h$ inherit Lipschitz continuity in their first argument from $f$ and $h$ is injective onto its image.
This decoupling assumption, made for ease of exposition, is not strictly necessary, and the theoretical results in this paper can be easily adapted to the coupled case.

Throughout our analysis, we will use the notation $\bx_{x,\hat\D}^{\bu,\bdelta}(\cdot)$ to denote the state trajectory $t\mapsto x$ corresponding to the initial condition $x\in\RR^n$, the control signal $\bu\in\UU$ and the disturbance signal $\bdelta\in\DD$, subjecting the latter to satisfy $\bdelta(t)\in\hat\D\big(\bx^{\bu,\bdelta}_{x,\hat\D}(t)\big)$ for all $t\ge0$. 
\subsection{State Constraints}\label{subsec:formulation_constraints}
A central element in our problem is the \emph{constraint set}, which defines a region $\K\subseteq \RR^n$ of the state space, typically 
the complement of all unacceptable failure states,
where the system is required to remain throughout the learning process.
This set is assumed closed and time-invariant; no further assumptions (boundedness, connectedness, convexity, etc.) are~needed.

From closedness,
we can implicitly characterize $\K$ as the zero superlevel set of a Lipschitz \emph{surface function} $l:\RR^n\rightarrow\RR$: 
\begin{equation}\label{eq:l}
x\in\K\iff l(x)\ge0.
\end{equation}
This function always exists, since we can simply choose $l(x) = -s_{\K}(x)$, 
which is Lipschitz continuous by definition. 

To express whether a given trajectory \emph{ever} violates the constraints, let the functional $\V:\RR^n\times\UU\times\DD\to\RR$ assign to each initial state $x$ and input signals $\bu(\cdot)$, $\bdelta(\cdot)$ the lowest value of $l(\cdot)$ achieved by trajectory $\bx_{x,\hat\D}^{\bu,\bdelta}(\cdot)$ over all times $t\ge0$: 
\begin{equation}\label{eq:V}
\mathcal{V}\big(x,\bu(\cdot),\bdelta(\cdot)\big) := \inf_{t\ge 0}l\big(\bx_{x,\hat\D}^{\bu,\bdelta}(t)\big).
\end{equation}
This outcome $\V$ will be strictly smaller than zero if there exists any $t\in[0,\infty)$ at which the trajectory leaves the constraint set, and will be nonnegative if the system remains in the constraint set for all of $t\ge 0$. Denoting $\V^{\bu,\bdelta}(x) = \V\big(x,\bu(\cdot),\bdelta(\cdot)\big)$, the following statement follows from \eqref{eq:l} and \eqref{eq:V} by construction.
\begin{proposition}\label{Value}
The set of points $x$ from which the system trajectory $\bx^{\bu,\bdelta}_{x,\hat\D}(\cdot)$ under given inputs $\bu(\cdot)\in\UU,\bdelta(\cdot)\in\DD$ will remain in the constraint set $\K$ at all times $t\ge0$ is equal to the zero superlevel set of $\mathcal{V}^{\bu,\bdelta}(\cdot)$:
\[
\{x\in\RR^n: \forall t\ge0,\; \bx^{\bu,\bdelta}_{x,\hat\D}(t)\in\K\}=\{x\in\RR^n: \mathcal{V}^{\bu,\bdelta}(\cdot)\ge0\}.
\]
\end{proposition}

Guaranteeing safe evolution from a given point $x\in\RR^n$ given an uncertainty bound $\hat\D$ requires determining whether there exists a control input $\bu(\cdot)\in\UU$ such that, for all disturbance inputs $\bdelta(\cdot)\in\DD$ satisfying $\bdelta(t)\in\hat\D\big(\bx^{\bu,\bdelta}_{x,\hat\D}(t)\big)$, the evolution of the system remains in $\K$, or equivalently $\V^{\bu,\bdelta}(x)\ge 0$. In Section \ref{sec:analysis}, we review how to answer this question using differential game theory and state some important properties of the associated solution. 
\subsection{Objective: Safe Learning}

Learning-based control aims to achieve desirable system behavior
by autonomously improving a policy $\kappa_l:{\RR^n\to\U}$, typically seeking to optimize an objective function.
Safe learning additionally requires that certain constraints $\K$ remain satisfied while searching for such a policy.
With full knowledge of the system dynamics ${F(x,u) = f\big(x,u,d(x)\big)}$, we would like to find 
a safe control policy $\kappa^*:\RR^n\to\U$ producing trajectories $\bx(t) \in \K$, $\forall t \ge 0$, for the largest set of initial states $x = \bx(0)$,
then restrict any learned policy so that $\kappa_l(x)=\kappa^*(x)$ wherever required to ensure safety.
When $d(x)$ is not known exactly, however, this problem cannot be solved.

Instead, given an estimated disturbance set $\hat \D(x)$, we can find an inner approximation of the set of safe states by considering all the possible trajectories that can be produced under the bounded uncertainty $d(x)\in\Dx$.
Our goal, then, is to find the set of \emph{robustly safe} states $x$ for which there exists a 
control policy $\kappa^*$ that can keep
the closed-loop system evolution
in $\K$, 
and consistently limit $\kappa_l$ to ensure that $\kappa^*$ is applied when necessary. 
To formally state this, we introduce an important notion from robust control theory.
\begin{definition}
A subset $\M\subset\RR^n$ is a \emph{robust controlled invariant set} under uncertain dynamics ${\dot x = f(x,u,d)}$, ${d\in\Dx}$, if
there exists a feedback control policy ${\kappa: \RR^n\to\U}$ such that
all possible system trajectories starting at $\bx(0)\in\M$ are guaranteed to satisfy
$\bx(t)\in\M$ for all time $t\ge0$.
\end{definition}

Given that trajectories are continuous, the system state can only leave $\M$ by crossing its boundary $\partial \M$. Hence if $\M$ is closed, applying the feedback policy $\kappa(x)$ for $x\in\partial\M$ is enough to render $\M$ robust controlled invariant, allowing an arbitrary control action to be applied in the interior of $\M$.

\begin{definition}\label{def:safe_set}
The \emph{safe set} $\Omega_{\hat\D}$ is the maximal robust controlled invariant set 
under uncertain dynamics ${\dot x = f(x,u, d)}$, ${d\in\Dx}$, that is contained in the constraint set $\K$.
\end{definition}

\new{Success in safe learning therefore seems closely linked to model uncertainty:}
a tighter bound $\hat \D(x)$ on $d(x)$ yields a less conservative safe set $\Omega_{\hat\D}$, which in turn reduces the restrictions on the learning process.
However, an estimated bound that fails to fully capture $d(x)$ may allow the system to execute control actions resulting in a constraint violation.
The disturbance bound should thus be as tight as possible, to allow the system greater freedom in learning, yet wide enough to confidently capture the unknown dynamics, in order to ensure safety. 

In the following two sections, we formalize this tradeoff and propose a framework to reason about safety guarantees under uncertainty. Section \ref{sec:analysis} poses the safety problem as a differential game between the controller and an adversarial disturbance, presenting a stronger result than commonly used in the reachability safety literature, which exploits the entire value function of the game rather than only its zero level set. Section \ref{sec:solution} leverages this result to provide a principled approach to global safety under model uncertainty, as well as a fast local alternative that may often be useful in practice. 
	
\section{Safety as a Differential Game}\label{sec:analysis}
The safety problem can be posed as a two-player zero-sum differential game between the system controller and the disturbance. Intuitively, we are requiring the controller to keep the system from violating the constraints for \emph{all} possible disturbance inputs within a certain family:
by conducting a worst-case analysis assuming an optimally adversarial disturbance, we implicitly protect the system against all ``suboptimal" disturbances as well.
\new{We first introduce relevant background on differential games, and then present new enabling insights.}
\new{\subsection{Background: HJI Equation and Safe Set}}
\label{subsec:analysis_safeset}

To obtain a \emph{safe set} and an associated \emph{safety policy},
we formulate a game whose outcome is given by the functional $\V\big(x,\bu(\cdot),\bdelta(\cdot)\big)$ introduced in \eqref{eq:V}, negative for those trajectories $\bx_{x,\hat\D}^{\bu,\bdelta}(\cdot)$ that at some point violate the constraints $\K$.

In the robust safety problem, the controller seeks to maximize the outcome of the game, while the disturbance tries to minimize it: that is, the disturbance is trying to drive the system out of the constraint set, and the controller wants to prevent it from succeeding. 
Following \cite{
Evans1984}, we define the set of \emph{nonanticipative strategies} for the disturbance containing the functionals
$\B = \{\bbeta:\UU\to\DD\;|\;
\forall t\ge 0,\; \forall \bu(\cdot),\hat{\bu}(\cdot)\in\UU,$
${\big(\bu(\tau) \!=\! \hat{\bu}(\tau)\text{ a.e.} \tau\ge0\big)}\Rightarrow{
{\big(\bbeta[\bu](\tau) \!=\! \bbeta[\hat{\bu}](\tau)}{\text{ a.e.} \tau\ge0\big)}}\}$. 
Since the disturbance and the control inputs are decoupled in the system dynamics, Isaacs' minimax condition holds\footnote{This means that we could have alternatively let the controller use nonanticipative strategies, without affecting the solution of the game.} and the \emph{value} of the game is well defined as:
\begin{equation}\label{eq:value}
{V}(x):=\inf_{\bbeta[\bu](\cdot)\in\B}\sup_{\bu(\cdot)\in\UU}\mathcal{V}\big(x,\bu(\cdot),\bbeta[\bu](\cdot)\big)
\enspace .
\end{equation}
Under this information structure, we draw on the (infinite-horizon) discriminating kernel concept from viability theory.
\begin{definition}\label{def:disc}
A point $x\in\K$ is in $\K$'s \emph{discriminating kernel} $Disc_{\hat\D}(\K)$ 
if the system trajectory $\bx^{\bu,\bdelta}_{x,\hat\D}$ starting at $x$, with both players acting optimally, remains in $\K$ for all time $t\ge0$:
\begin{align*}
\Disc_{\hat\D}(\K) := 
\{&x\in\RR^n: \forall \bbeta(\cdot)\in\B, \exists \bu(\cdot)\in\UU
    \enspace,\notag \\
&\forall t\ge0, \bx^{\bu,\bbeta[\bu]}_{x,\hat\D}(t)\in\K\}
    \enspace .
\end{align*}
\end{definition}
The following classical result follows from Proposition \ref{Value}.
\begin{proposition}
The discriminating kernel of the constraint set $\K$
is the zero superlevel set of the value function $V$:
\begin{align*}
\Disc_{\hat\D}(\K)&= \{x\in\RR^n: V(x)\ge0\}.
\end{align*}
\end{proposition}

Further, from Definitions \ref{def:safe_set} and \ref{def:disc}, it can be seen that the discriminating kernel $\Disc_{\hat\D}(\K)$ is identical to the safe set $\Omega_{\hat\D}$.

It has been shown that the value function for \emph{minimum payoff} games of the form presented in Section \ref{subsec:analysis_safeset}
(i.e. games in which the payoff is the minimum of a state function over time)
can be characterized as the unique viscosity solution to a variational inequality involving an appropriate Hamiltonian \cite{Barron1990},~\cite{Fisac2015};
an alternative formulation involves a modified partial differential equation \cite{Mitchell2005}.
In a finite-horizon setting, with the game taking place over the compact time interval $[0,T]$,
the value function $V(x,t)$
can be computed by solving the Hamilton-Jacobi-Isaacs (HJI) variational inequality:
\begin{subequations}\label{eq:HJI}\begin{align}
    &\fixwidth{
    \min \! \left\{ \!
       l(x) \! - \! V(x,t),
       \frac{\partial V}{\partial t}(x,t) \! + \!
       \max_{u\in\U} \!\min_{ d\in\hat\D(x)} \!\! \frac{\partial V}{\partial x}(x,t) f(x,u, d)
    \! \right\}
    }\label{eq:HJIa}\\
    &V(x,T) = l(x).\label{eq:HJIb}
\end{align}\end{subequations}

As long as there exists a nonempty safe set in the problem,
$V(x,t)$ becomes independent of $t$ inside of this set
as $T\to\infty$.
We accordingly drop the dependence on $t$ and recover $V(x)$ as defined in \eqref{eq:value}, which we refer to as the \emph{safety function}.

\begin{definition}
The \emph{optimal safe policy} $\kappa^*(\cdot)$ is the solution to the optimization:%
\footnote{
    While in general the solution need not be unique, we can always choose one element of the $\arg\max$ set arbitrarily. Therefore we will assume for simplicity a policy $\kappa^*:\RR^n\to\U$ uniquely mapping states to control inputs.
}
\[ \kappa^*(x) = \arg\max_{u\in\U} \min_{ d\in\Dx} \frac{\partial V}{\partial x}(x) f(x,u, d).\]
\end{definition}
Policy $\kappa^*(x)$ 
attempts to drive the system to the safest possible state always assuming an adversarial disturbance.
If the disturbance bound $\hat{D}(x)$ is correct everywhere,
then one can allow the system to execute any desired control
while in the interior of $\Omega_{\hat\D}$, as long as the safety preserving action $\kappa^*(x)$ is taken whenever the state reaches the boundary $\partial \Omega_{\hat\D}$; the system is then guaranteed to remain inside $\Omega_{\hat\D}$ for all time.
This least-restrictive control law can be used in conjunction with an arbitrary learning-based control policy $\kappa_l(x)$ (which may be repeatedly updated by the corresponding learning algorithm), to produce a \emph{safe learning policy}:

\begin{equation}\label{eq:least-restrictive}
\kappa(x) =
\begin{cases}
\kappa_l(x), & \text{if $V(x)>0$}, \\
\kappa^*(x), & \text{otherwise}.
\end{cases}
\end{equation}

\new{%
    Rather than imposing the optimal safe action $\kappa^*(x)$,
    it would have, in principle, been sufficient to
    project the desired $\kappa_l(x)$ onto the set of control inputs that guarantee nonnegative local evolution of $V$ for all $d\in\Dx$.
    However, $\kappa^*(x)$
    results in the greatest predicted increase in value, which is desirable under model uncertainty, as we will see in Section \ref{subsec:validation}.
}%
\subsection{ Invariance Properties of Level Sets}

Traditionally, the implicit hypothesis made to guarantee safety using a least-restrictive law in the form of \eqref{eq:least-restrictive} has been correctness of the estimated disturbance bound $\hat\D$ everywhere in the state space,
(i.e. $d(x)\in\hat{\D}(x)$ $\forall x\in\RR^n$), or at least everywhere in the constraint set $\K$ \cite{Mitchell2005},~\cite{Gillula2012a}. We will now argue that the necessary hypothesis for safety is in fact much less stringent, by proving an important result that we will use in the following section to strengthen the proposed safety framework and retain safety guarantees under partially incorrect models.
\begin{proposition}\label{prop:level_set}
    Any nonnegative superlevel set of $V(x)$ is a robust controlled invariant set with respect to $ d\in\hat{\D}(x)$.
\end{proposition}
\begin{proof}
By Lipschitz continuity of $f$ and $l$, we have that $V$ is Lipschitz continuous \cite{Evans1984} and hence, by Rademacher's theorem, almost everywhere differentiable.
The convergence of $V(x,t)$ to $V(x)$ as $T\to\infty$ implies that at the limit $\frac{\partial V}{\partial t}(x,t)=0$.
Therefore, given any $\alpha\geq 0$, for any point $x \in \{x \mid V(x) \geq \alpha\}$ there must exist a control action $u^*$ such that $\forall  d\in\Dx$, $\frac{\partial V}{\partial x}(x) f(x,u^*, d) \geq 0$; otherwise the right hand side of~\eqref{eq:HJIa} would be strictly negative for $T\to\infty$, contradicting convergence. Then, the value of $V$ from any such state $x$ can always be kept from decreasing, so $\{x|V(x)\ge\alpha\}$ is a robust controlled invariant set with respect to $ d\in\Dx$.
\end{proof}
\begin{proposition} \label{prop:invariance}
    Consider two disturbance sets $\mathcal{D}_1(x)$ and $\mathcal{D}_2(x)$, and a closed set $\M\subset \RR^n$ that is robustly controlled invariant under $\mathcal{D}_1(x)$. If $\mathcal{D}_2(x) \subseteq \mathcal{D}_1(x)$ $\forall x \in \partial \M$, then $\M$ is robustly controlled invariant also under $\mathcal{D}_2(x)$. 
\end{proposition}
\begin{proof}
    Consider an arbitrary trajectory $\bx^{\bu,\bdelta}_{x_0,\D_2}\in\XX_{\D_2}$ under the disturbance set $\D_2(x)$, starting at $x_0\in\M$, such that for some $\tau<\infty$, $\bx(\tau)\not\in\M$.
    Since trajectories in $\XX_{\D_2}$ are continuous, there must then exist $s \in[t_0,\tau]$ such that $\bx^{\bu,\bdelta}_{x_0,\D_2}(s) \in \partial \M$. On the other hand, because $\M$ is robustly controlled invariant under $\D_1(x)$, we know that $\exists \kappa:\RR^n\to\U$ such that no possible disturbance $d\in\D_1(x)$ can drive the system out of $\M$. Since $\D_2(x) \subseteq \D_1(x)$ $\forall x \in \partial \M$, the same control policy $\kappa^*(x)$ on the boundary guarantees that no disturbance $d\in\D_2(x)\subseteq\D_1(x)$ can drive the system out of $\M$. Hence, for $\bx^{\bu,\bdelta}_{x_0,\D_2}$, switching to policy $\kappa$ at time $s$ guarantees that the system will remain in $\M$. Therefore $\M$ is a robust controlled invariant set under $\D_2(x)$.
\end{proof}
\begin{corollary} \label{cor:invariance}
    Let $\Q_\alpha=\{x\in\RR^n: V(x)=\alpha\}$ with $\alpha\ge0$ be any nonnegative level set of the safety function $V$, computed for some disturbance set $\hat\D(x)$. If $d(x)\in \hat\D(x)$ $\forall x\in\Q_\alpha$, then the superlevel set $\{x\in\RR^n: V(x)\ge\alpha\}$ is an invariant set under the computed safe control policy $\kappa^*(x)$.
\end{corollary}

This corollary, which follows from Propositions \ref{prop:level_set} and \ref{prop:invariance} by considering the singleton $\{d(x)\}$, is an important result that will be at the core of our data-driven safety enhancement.
It provides a sufficient condition for safety, but unlike the standard HJI solution, it does not readily prescribe a least-restrictive control law to exploit it:
how should one determine what candidate $\alpha\ge0$ to choose, or whether a valid $\Q_\alpha$ exists at all?
Deciding when the safe controller should intervene and what guarantees are possible is nontrivial and requires additional analysis.

The next section proposes a Bayesian approach enabling the safety controller to reason about its confidence in the model-based guarantees described in this section.
If this confidence reaches a prescribed minimum value in light of the observed data, the controller can intervene early to ensure that safety will be maintained with high probability.

\section{Bayesian Safety Analysis}
\label{sec:solution}
\subsection{Learning-Based Safe Learning}

As we have seen, robust optimal control and dynamic game theory provide powerful analytical tools to study the safety of a dynamical model. However, it is important to realize that the applicability of any theoretically derived guarantee to the real system is contingent upon the validity of the underlying modeling assumptions;
in the formulation considered here, this amounts to the state disturbance function $d(x)$ being captured by the bound $\hat\D(x)$
on at least a certain subset of the state space.
The system designer therefore faces an inevitable tradeoff between risk and conservativeness, due to the impossibility of accounting for every aspect of the real system in a tractable model.

In many cases, choosing a parametric model \emph{a priori} forces one to become overly conservative in order to ensure that the system behavior will be adequately captured: this results in a large bound $\hat\D(x)$ on the disturbance, which typically leads to a small safe set $\Omega_{\hat\D}$, limiting the learning agent's ability to explore and perform the assigned tasks. In other cases, insufficient caution in the definition of the model can lead to an estimated disturbance set $\hat\D(x)$ that fails to contain the actual model error $d(x)$, and therefore the computed safe set $\Omega_{\hat\D}$ may not in fact be controlled invariant in practice, which can end all safety guarantees.

In order to avoid excessive conservativeness and keep theoretical guarantees valid, it is imperative to have both a principled method to refine the system model based on acquired measurements and a reliable mechanism to detect and react to model discrepancies with the real system's behavior; both of these elements are necessarily data-driven. We thus arrive at what is perhaps the most important insight in this work: \emph{the relation between safety and learning is reciprocal}. Not only is safety a key requirement for learning in autonomous systems: learning about the real system's behavior is itself indispensable to provide practical safety guarantees.

\new{%
In the remainder of this section we propose a method for reasoning about the uncertain system dynamics, using Gaussian processes
    to regularly update the model used for safety analysis,
and introduce a Bayesian approach for online validation of model-based guarantees
\emph{in between} updates.
}%
We then define an adaptive safety control strategy based on this real-time validation, which leverages the theoretical results from Hamilton-Jacobi analysis to provide stronger guarantees
for safe learning under possible model inaccuracies. 
\subsection{Gaussian Process}\label{subsec:GP}

To estimate the disturbance function $d(x)$ over the state space, we model it as being drawn from a Gaussian process. Gaussian processes are a powerful
abstraction that
extends multivariate Gaussian regression to the infinite-dimensional space of functions, allowing Bayesian inference based on (possibly noisy) observations of a function's value at finitely many points.%
\footnote{
    We give here an overview of Gaussian process regression and direct the interested reader to \cite{Rasmussen2006} for a more comprehensive introduction.
}

A Gaussian process is a random process or field defined by a mean function ${\mu:\RR^n \rightarrow \RR}$ and a positive semidefinite covariance kernel function ${k:\RR^n \times \RR^n \rightarrow \RR}$.  We will treat each component ${d^j, j\in\{1,...,n_d\}}$, of the disturbance function as an independent Gaussian process:
\begin{equation}\label{eq:GP}
d^j(x)\sim \mathcal{GP}(\mu^j(x),k^j(x,x')).
\end{equation}
A defining characteristic of a Gaussian process is that the marginal probability distribution of the function value at any finite number of points is a multivariate Gaussian.
This will allow us to
obtain the disturbance bound $\Dx$ as a Cartesian product of confidence intervals
for the components of $d(x)$ at each state $x$, choosing the bound to capture a desired degree of confidence.%
\footnote{
\new{
    By assuming independence of disturbance components
    we are effectively over-approximating the confidence ellipsoid in $\RR^{n_d}$ by its minimal containing box; a less conservative analysis could compute $\Dx$ using a vector-valued Gaussian process model, at the expense of heavier computation.
}
}

Gaussian processes allow incorporating new observations in a nonparametric Bayesian setting. First, assume a prior Gaussian process distribution over the $j$-th component of $d(\cdot)$, with mean $\mu^j(\cdot)$ and covariance kernel $k^j(\cdot,\cdot)$. The class of the prior mean function and covariance kernel function is chosen to capture the characteristics of the model (linearity, periodicity, etc), and is associated to a set of hyperparameters~$\theta_p$. These are typically set to maximize the marginal likelihood of an available set of training data, or possibly to reflect some prior belief about the system.

Next, consider $N$ measurements ${\bf  \hat{d}}^j= [ \hat{d}^j_1,\hdots,  \hat{d}^j_N ]$, observed with independent Gaussian noise $\epsilon_i^j\sim\N(0,(\sigma_n^j)^2)$ at the points $X=[ x_1 , \hdots , x_N ] $, i.e. $ \hat{d}^j_i=d^j(x_i)+\epsilon_i^j$. 
Combined with the prior distribution \eqref{eq:GP}, this new evidence induces a Gaussian process posterior; in particular,
the value of $d^j$
at finitely many points $X_*$
is distributed as a multivariate normal:%
\begin{subequations} \label{eq:posterior}
\begin{align}
&\mathbb{E}[d^j(X_*)\mid {\bf  \hat{d}}^j, X] = \\
& \fixwidth{ \mu^j(X_*) + K^j(X_*,X)(K^j(X,X)+(\sigma_n^j)^2I)^{-1}  (\mathbf{\hat{d}}^j-\mu^j(X)),} \notag
\end{align}%
\begin{align}
&\text{cov}[d^j(X_*) \mid X] =\\
& \fixwidth{ K^j(X_*,X_*)-K^j(X_*,X)(K^j(X,X)+(\sigma_n^j)^2I)^{-1} K^j(X,X_*), } \notag
\end{align}
\end{subequations}
where $d^j_i(X)=d^j(x_i)$, $\mu^j_i(X)=\mu^j(x_i)$, and for any $X,X'$ the matrix $K^j(X,X')$ is defined component-wise as $K^j_{ik}(X,X') = k^j(x_i,x_k')$. Note that whenever a new batch of data $X$ is obtained the hyperparameters of the kernel function are refitted, so the variance implicitly depends on the measurements $d^j$.
If a single query point is considered, i.e. $X_*=\{x_*\}$, the marginalized Gaussian process posterior becomes a univariate normal distribution quantifying both the expected value of the disturbance function, $\bar{d^j}(x_*)$, and the uncertainty of this estimate, $\big(\sigma^j(x_*)\big)^2$, 
\begin{subequations}  \label{eq:single}
\begin{align} &\bar{d^j}(x_*)=\mathbb{E}[d^j(x_*)| {\bf  \hat{d}}^j, X] 
\\
&\big(\sigma^j(x_*)\big)^2=\text{cov}[d^j(x_*)|X]
\enspace .
\end{align}\end{subequations}

We can use the Bayesian machinery of Gaussian process regression to compute a \emph{likely} bound $\Dx$ on the disturbance function $d(x)$ based on the history of commanded inputs $u_i$ and state measurements $x_i$, $i\in\{1,...,N\}$.
To this effect, we assume that a method for approximately measuring the state derivatives is available (e.g. by numerical differentiation), and denote each of these measurements by $\hat f_i$. Based on \eqref{eq:decoupled}, we can obtain measurements of $d(x_i)$
from the residuals between the observed dynamics and the model's prediction:
\begin {equation} \label{eq:residual}
\hat{d}(x_i)= h^{-1}\left( \hat f_i -  g\big(x_i,u_i\big) \right).
\end {equation}
The residuals $\mathbf{\hat{d}}= [ \hat{d}(x_1), \hdots,  \hat{d}(x_N) ]$ are processed through \eqref{eq:posterior} to infer the marginal distribution of $d(x_*)$ for an arbitrary point $x_*$, specified by the expected value $\bar{d^j}(x_*)$ and the standard deviation $\sigma^j(x_*)$ of each component of the disturbance. This distribution can be used to construct a disturbance set $\hat{\D}(x_*)\subseteq\D$ at any point $x_*$; in practice, this will be done at finitely many points $x_{\bf{i}}$ on a grid, and used in the numerical reachability computation to obtain the safety function and the safe control policy.

We now introduce the design parameter $p$ as the desired marginal probability that the disturbance function $d(x)$ will belong to the bound $\hat{\D}(x)$ at each point $x$;
typically, $p$ should be chosen to be close to 1. 
The set $\hat{\D}(x)$ is then chosen for each $x$
as follows. Let $z = \sqrt{2}\,\text{erf}^{-1}(p^{1/n_d})$, where $\text{erf}(\cdot)$ denotes the Gauss error function; that is, define $z$ so that the probability that a sample from a standard normal distribution $\N(0,1)$ lies within $[-z,z]$ is $p^{1/n_d}$. We construct $\Dx$ by taking a Cartesian product of confidence intervals:
\begin{equation}\label{eq:Dx}
\hat{\D}(x)= \prod_{j=1}^{n_d}[\bar{d^j}(x )-z\sigma^j(x), \, \bar{d^j}(x )+z\sigma^j(x)] .
\end{equation}
Since each component $d^j(x)$ is given by an independent Gaussian $\N\left(\bar d^j(x), \sigma^j(x)\right)$, the probability of $d(x)$ lying within the above hyperrectangle is by construction $\left(p^{1/n_d}\right)^{n_d}=p$.
\begin{remark}
It is commonplace to use Gaussian distributions to capture beliefs on variables that are otherwise known to be bounded. While one might object that the unbounded support of \eqref{eq:single} contradicts our problem formulation (in which the disturbance $d$ took values from some compact set $\D\subset\RR^{n_d}$), the hyperrectangle $\Dx$ in \eqref{eq:Dx} is always a compact set. Note that the theoretical input set $\D$ is never needed in practice, so it can always be assumed to contain $\Dx$ for all $x$.
\end{remark}

Under Lipschitz continuous prior means $\mu^j$ and covariance kernels $k^j$, the disturbance bound \eqref{eq:Dx} varies (Hausdorff) Lipschitz-continuously in $x$, satisfying the hypotheses of Proposition \ref{prop:locally_correct}. This is formalized and proved in the Appendix.

The safety analysis described in Section \ref{sec:analysis} can be carried out by solving the Hamilton-Jacobi equation \eqref{eq:HJI} for $\Dx$ given by \eqref{eq:Dx}, which---based on the information available at the time of computation---will be a correct disturbance bound at any single state $x$ with probability $p$. As the system goes on to gather new information, however, the posterior probability for $d(x)\in\Dx$ will change at each $x$ (and will typically no longer equal $p$). More generally, we have the following result.
\begin{proposition}\label{prop:local_climb}
Let $q$ be the probability that $d(x)\in\hat\D(x)$ for some state $x$ with $V(x)\ge0$. Then the probability that $L_f^*V(x):=D_xV(x)\cdot f\left(x,\kappa^*(x),d(x)\right)\ge0$ is at least $q$.
\begin{proof}
Omitting $x$ for conciseness, we have: $P( L_f^*V\ge0) = P\big(L_f^*V\!\ge\!0|d\in\hat\D\big) P\big(d\in\hat\D\big) + P\big(L_f^*V\!\ge\!0|d\not\in\hat\D\big) P\big(d\not\in\hat\D\big).$

By Corollary \ref{cor:invariance}, the first term evaluates to $1\cdot q$; the second term is nonnegative (and will typically be positive, since not all values of $d\not\in\hat\D$ will be unfavorable for safety, and there may be some for which the input $\kappa^*(x)$ leads the system to locally increase $V$).
\end{proof}
\end{proposition}

Based on this result, we can begin to reason about the guarantees of the reachability analysis applied to the real system in a Bayesian framework, inherited from the Gaussian process model. 
\subsection{Online Safety Guarantee Validation}\label{subsec:validation}

In order to ensure safety under the possibility of model mismatch with the real system, it may become necessary to intervene not only on the boundary of the computed safe set, but also whenever the observed evolution of the system indicates that the model-based safety guarantees may lose validity. Indeed, failure to take a safe action in time may lead to complete loss of guarantees if the system enters a region of the state space where the model is consistently incorrect.

While the estimated bound $\hat\D$ (Section \ref{subsec:GP}) and the associated safety guarantees (Section \ref{sec:analysis}) should certainly be recomputed as frequently as possible in light of new evidence, this process can typically take seconds or minutes, and in some cases may even require offline computation.
\new{%
    This motivates the need to augment model-based guarantees through
    an online data-driven
    mechanism
    to quickly adapt to new incoming information%
even as new, improved guarantees are computed.
}%

Bayesian analysis allows us to update our belief on the disturbance function as new observations are obtained. This in turn can be used to provide a probabilistic guarantee on the validity of the safety results obtained from the robust dynamical model generated from the older observations. In the remainder of this section, we will discuss how to update the belief on the disturbance function, and then provide two different theoretical criteria for safety intervention. The first criterion provides global probabilistic guarantees, but has computational challenges associated to its practical implementation. The alternative method only provides a local guarantee, but can more easily be applied in real time.

Let us denote $X_\text{old}$ and ${\bf{\hat{d}^j}}_\text{old}$ as the evidence used in computing the 
disturbance set $\hat{\D}(x)$, and $X_\text{new}$ and $\bf{\hat{d}^j}_\text{new}$ as the evidence acquired online after the disturbance set is computed. Conditioned on the old evidence, the function $d^j(x)$ is normally distributed with mean and variance given by \eqref{eq:posterior} with $X=X_\text{old}$ and $\bf{\hat{d}^j}=\bf{\hat{d}^j}_\text{old}$, and the disturbance set is given by  \eqref{eq:Dx}. If we also condition on 
the new evidence and keep the hyperparameters fixed, then the mean and variance are updated by modifying \eqref{eq:posterior}  
with  $X=[X_\text{old}, X_\text{new}]$ and $\bf{\hat{d}^j}=[\bf{\hat{d}^j}_\text{old},\bf{\hat{d}^j}_\text{new}]$. 

\begin{remark}
Performing the update requires inverting $K^j([X_\text{old},X_\text{new}],[X_\text{old},X_\text{new}])$.
This can be done efficiently employing standard techniques:
since 
$K^j(X_{old},X_{old})$ has already been inverted (in order to compute the disturbance bound $\hat\D$), all that is needed is inverting the Schur Complement of  $K^j(X_\text{old},X_\text{old})$ in $K^j([X_\text{old},X_\text{new}], [X_\text{old},X_\text{new}])$, which has the same size as
$K^j(X_\text{new},X_\text{new})$. 
\end{remark}
\new{\indent
    Ideally we would incorporate $X_\text{new}$ and $\bf{\hat{d}^j}_\text{new}$ to relearn the Gaussian process hyperparameters as quickly as new measurements come in:
    otherwise new measured disturbance values $\bf{\hat{d}^j}_\text{new}$ will only affect the posterior mean, with the variance depending exclusively on where the measurements were made ($X_\text{new}$).
    However, performing this update online is computationally prohibitive.
    Instead, we update the hyperparameters every time a new estimated bound $\hat\D$ is produced for safety analysis, keeping them fixed in between.
    In practice the set $X_\text{old}$ will be much larger than $X_\text{new}$, so the estimated hyperparameters would not be expected to change significantly.%
}%
\begin{remark}
    \new{%
    In settings where conditions are slowly time-varying, it may be desirable to give recently observed data more weight than older observations.
    This can naturally be encoded by the Gaussian process by appending time as an additional dimension in $X$:
    points that are distant in time would then be more weakly correlated, analogous to space.%
    }
\end{remark}

Based on the new Gaussian distribution, 
we can reason about the \emph{posterior} confidence in the safety guarantees produced by our original safety analysis,
which relied on the \emph{prior} Gaussian distribution resulting from measurements $\bf{\hat{d}^j}_\text{old}$ at states $X_\text{old}$.

\subsubsection{Global Bayesian safety analysis}

The strongest result available for guaranteeing safety under the present framework is Corollary \ref{cor:invariance}, which allows the system to exploit any superzero level set $\Q_\alpha$ ($\alpha\ge0$) of the safety function $V$ throughout which the model is locally correct; all that is needed is for such a $\Q_\alpha$ to exist for $\alpha\in[0,V(x)]$ given the current state $x$.

It is possible to devise a safety policy to fully exploit the sufficient condition in Corollary \ref{cor:invariance} in a Bayesian setting:
if the posterior probability that the corollary's hypotheses will hold drops to some arbitrary \emph{global confidence threshold} $\gamma_0$, the safe controller can override the learning agent.
With probability $\gamma_0$, the corollary will still apply, in which case the system is guaranteed to remain safe for all time;
even if Corollary \ref{cor:invariance} does not apply at this time (which could happen with probability $1-\gamma_0$), it is still possible that the disturbance $d(x)$ will not consistently take adversarial values that force the computed safety function $V(x)$ to decrease, in which case the system may still evolve safely. Therefore, this policy guarantees a lower bound on the probability of maintaining safety for all time.

In order to apply this safety criterion, the system needs to maintain a Bayesian posterior of the sufficient condition in Corollary \ref{cor:invariance}. We refer to this posterior probability as the \emph{global safety confidence} $\gamma(x;X,{\bf \hat{d}^j})$, or $\gamma(x)$ for conciseness:
\begin{equation}\label{eq:gp_bound_joint}
\begin{split}
\gamma(x;X,{\bf \hat{d}^j}):=P\big(&\exists\alpha\in[0,V(x)],\forall x\in\Q_\alpha: d(x)\in\Dx
|X,{\bf \hat{d}^j}\big).
\end{split}
\end{equation}

Based on this, we propose the least-restrictive control law:
\begin{equation}\label{eq:global_strategy}
\kappa(x) =
\begin{cases}
\kappa_l(x), & \text{if } \big( \gamma(x) > \gamma_0 \big) \wedge \big( V(x)>0 \big),\\
\kappa^*(x), & \text{otherwise,}
\end{cases}
\end{equation}
so the system applies any action it desires if the global safety confidence is above the threshold, but applies the safe controller once this is no longer the case.

Note that if confidence in the safety guarantees is restored after applying the safety action the learning algorithm will be allowed to resume control of the system. This can happen by multiple mechanisms: moving to a region with higher $V(x)$ will tend to increase the probability that \emph{some} lower level set may satisfy the hypotheses of Corollary \ref{cor:invariance}; moving to a region with less inconsistency between expected and observed dynamics will typically lead to higher posterior belief that \emph{nearby} level sets will satisfy the hypotheses of Corollary \ref{cor:invariance}; and generally acquiring new data may, in some cases, increase the posterior confidence that Corollary \ref{cor:invariance} may apply.

    Computing the joint probability that the bound $\Dx$ captures the Gaussian process $d(x)$ \emph{everywhere} on a level set $\Q_\alpha$ is not
    possible, since
    the set of functions $d(x)$ satisfying this condition is bounded on uncountably many dimensions, and thus not measurable in function space.
Similarly, evaluating the joint probability for a continuum of level sets $\Q_\alpha$ for $\alpha\in[0,V(x)]$ is not feasible.
    Instead, exploiting the Lipschitz assumption on $d(x)$,
    we can obtain the sought probability $\gamma(x)$ from a marginal distribution
over a sufficiently dense set of sample points on each $\Q_\alpha$ and a sufficiently dense collection of level sets between $0$ and $V(x)$.

We can then use numerical methods \cite{Genz1992}
to compute the multivariate normal cumulative distribution function and estimate the marginal probability
\new{(using compact logic notation)}:
\begin{equation}\label{eq:gp_bound_marginal}
\gamma(x)\approx P\bigg(\bigvee_{s=1}^S \bigwedge_{i=1}^{I} d(x_{s,i})\in\hat\D(x_{s,i})\bigg),
\end{equation}
over $S$ level sets $0=\alpha_0< ...<\alpha_S=V(x)$ and $I$ sample points from each level set $Q_{\alpha_s}$.
As the density of samples increases with larger $S$ and $I$, the marginal probability \eqref{eq:gp_bound_marginal} asymptotically approaches the Gaussian process probability \eqref{eq:gp_bound_joint}. Unfortunately, however, current numerical methods can only efficiently approximate these probabilities for multivariate Gaussians of about 25 dimensions \cite{Genz1992},
which drastically limits the number of sample points ($S\times I\approx 25$) that the marginal probability can be evaluated over, making it difficult to obtain a useful estimate. In view of this, a viable approach may be to bound \eqref{eq:gp_bound_joint} below as follows:
\begin{equation}\label{eq:gp_lower_bound_joint}
\gamma(x)\ge \underline\gamma(x) := \max_{\alpha\in[0,V(x)]} P\big(\forall x\in\Q_\alpha: d(x)\in\Dx\big),
\end{equation}
and approximately compute this as
\begin{equation}\label{eq:gp_lower_bound_marginal}
\underline\gamma(x) \approx \max_{s\in\{1,...,S\}}P\bigg(\bigwedge_{i=1}^{I} d(x_{s,i})\in\hat\D(x_{s,i})\bigg),
\end{equation}
with the advantage that a separate multivariate Gaussian evaluation can be done now for each level set ($I\approx 25$). Computing this approximate probability as the system explores its state space provides a decision mechanism to guarantee safe operation of the system with a desired degree of confidence, which the system designer or operator can adjust through the $\gamma_0$ parameter. 

\subsubsection{Local Bayesian safety analysis}

Evaluating the expression in \eqref{eq:gp_lower_bound_marginal} is still computationally intensive, which can limit the practicality of this method for real-time validation of safety guarantees in some applications, such as mobile robots relying on on-board processing. An alternative is to replace the global safety analysis with a local criterion that offers much faster computation traded off with a weaker safety guarantee.

Instead of relying on Corollary \ref{cor:invariance}, this lighter method exploits Propositions \ref{prop:locally_correct} and \ref{prop:local_climb}. The system is allowed to explore the computed safe set freely as long as the probability of the estimated model $\hat\D$ being \emph{locally reliable} remains above a certain threshold $\lambda_0$; if this threshold is reached, the safe controller intervenes, and the system is guaranteed to locally maintain or increase the computed safety value $V(x)$ with probability no less than $\lambda_0$. While this local guarantee does not ensure safety globally, it does constitute a useful heuristic effort to prevent the system from entering unexplored and potentially unsafe regions of the state space. Further, although the method is not explicitly tracking the hypotheses of Corollary \ref{cor:invariance}, the local result becomes a global guarantee if these hypotheses do indeed hold.

We define the \emph{local safety confidence} $\lambda(x;X,{\bf \hat{d}^j})$, more concisely $\lambda(x)$, as the posterior probability that $d(x)$ will be contained in $\hat{\D}(x)$ at the current state $x$, given all observations made until now:
\begin{equation}\label{eq:lambda}
\lambda(x;X,{\bf \hat{d}^j}):=P\big(d(x)\in\hat{\D}(x)\mid {X, \bf \hat{d}}^j\big)
\end{equation}
We then have the following local safety certificate.

\begin{proposition}\label{prop:local_model_confidence}
Let the disturbance $d(\cdot)$ be distributed component-wise as $n_d$ independent Gaussian processes \eqref{eq:GP}.
The safety policy $\kappa^*(\cdot)$ is guaranteed to locally maintain or increase the system's computed safety $V(\cdot)$
with probability greater than or equal to the local safety confidence $\lambda(x)$.
\begin{proof}
The proof follows directly from Propositions \ref{prop:locally_correct} and \ref{prop:local_climb}, and the definition of $\lambda(x)$, noting that the boundary of $\Dx$ has zero Lebesgue measure and thus under any Gaussian distribution $P(d\in\interior\Dx|d\in\Dx)=1$.
\end{proof}
\end{proposition}

A \emph{local confidence threshold} $\lambda_0\in(0,p)$ can be established such that whenever $\lambda(x)<\lambda_0$
the model is considered insufficiently reliable (reachability guarantees may fail locally with probability greater than $1-\lambda_0$), and the safety control is applied. 
The proposed safety control strategy is therefore as follows:
\begin{equation}\label{eq:modstrategy}
\kappa(x) =
\begin{cases}
\kappa_l(x), & \text{if } \big(V(x) > 0\big) \wedge \big(\lambda(x) > \lambda_0\big),\\
\kappa^*(x), & \text{otherwise,}
\end{cases}
\end{equation}
Similarly to \eqref{eq:global_strategy}, under this control law, if confidence on the local reliability of the model is restored after applying the safe action and making new observations, the system will be allowed to resume its learning process, as long as it is in the interior of the computed safe set. 

After generating a new Gaussian process model and defining $\hat{\D}(x)$ as described in Section \ref{subsec:GP}, the prior probability with which the disturbance function $d(x)$ belongs to the set $\hat{\D}(x)$ is by design $p$ everywhere in the state space. As the system evolves, more evidence is gathered in the form of measurements of the disturbance along the system trajectory, so that the belief that $d(x)\in\hat{\D}(x)$ is updated for each $x$. In particular, in the Gaussian process model, this additional evidence amounts to augmenting the covariance matrix $K^j$ in \eqref{eq:posterior} with additional data points and reevaluating the mean and variance of the posterior distribution of $d(x)$. 
Based on the new Gaussian distribution, $\lambda(x;X,{\bf \hat{d}^j})$ can readily be evaluated for each $x$ as
\begin{equation} \label{eq:erf}
\fixwidth{\lambda(x)=
\!\prod_{j=1}^{n_d}\frac{1}{2}\left[\text{erf}\left(\frac {d_+^j(x)-m^j(x)}{s^j(x)\sqrt{2}}\right)-\text{erf}\left(\frac {d_-^j(x)-m^j(x)}{s^j(x)\sqrt{2}}\right)\right],}
\end{equation}
with ${d_+^j(x)=\bar{d}^j(x )+z\sigma^j(x)}$, ${d_-^j(x)=\bar{d}^j(x )-z\sigma^j(x)}$, ${m^j(x)=\mathbb{E}[d^j(x)| X,{\bf  \hat{d}}^j]}$, ${s^j(x)=\sqrt{\text{var}(d^j(x) | X)}}$; recall that $z$ was defined to yield the desired probability mass $p$ in $\Dx$ at the time of safety computation, as per \eqref{eq:Dx}.

Parameters $p$ and $\lambda_0$ (or, in its case, $\gamma_0$) allow the system designer to choose the degree of conservativeness in the system: while $p$ regulates the amount of uncertainty accounted for by the robust model-based safety computation, $\lambda_0$ ($\gamma_0$) determines the acceptable degradation in the resulting certificate's posterior confidence before a safety intervention is initiated.
A value of $p$ close to 1 will lead to a large, high-confidence $\Dx$ throughout the state space, but this analysis may result in a small or even empty safe set; on the other hand, if $p$ is low, $\Dx$ will be smaller and the computed safe set will be larger, but guarantees are more likely to be deemed unreliable (as per $\lambda_0$ or $\gamma_0$) in light of later observations.

In the case of local safety analysis, immediately after computing a new model $\hat\D$, $\lambda(x)$ is by construction equal to $p$ everywhere in the state space. As more measurements are obtained, the posterior distribution over the disturbance changes, 
as illustrated in Fig. \ref{fig:disturbance_distr},
which can result in $\lambda(x)$ locally increasing or decreasing. If $\lambda_0$ is chosen to be close to $p$, it is likely that the safety override will take place under minor deviations with respect to the model's prediction; as $\lambda_0$ becomes lower, however, the probability that the disturbance will violate the modeling assumptions before the safety controller intervenes increases. This reflects the fundamental tradeoff between risk and conservativeness in safety-critical decision making under uncertainty. The proposed framework therefore allows the system designer to adjust the degree of conservativeness according to the needs and characteristics of the system at hand.

\begin{figure}
\begin{center}
\includegraphics[trim= .4cm 1cm .5cm 5.5cm, clip, width=.48\textwidth]{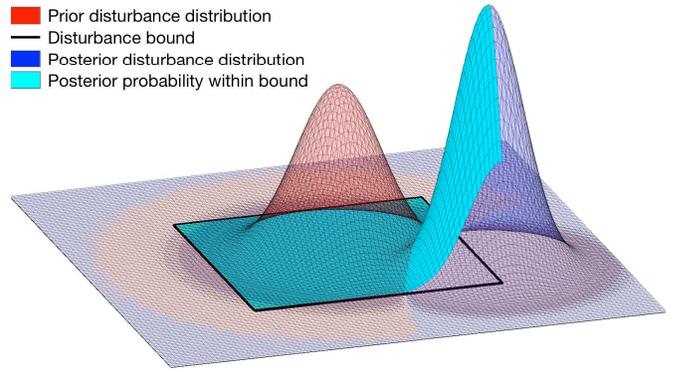}
\caption{Evolution of the probability distribution of the disturbance $d(x)$ at a particular state $x$. The prior distribution is used to compute the bound $\Dx$ using confidence intervals, such that it contains a specified probability mass $p$. As more data are obtained, the distribution may shift, leading to a different posterior probability mass contained within $\Dx$. \label{fig:disturbance_distr}}
\end{center}
\end{figure} 

\section{Experimental Results}\label{sec:results}
We demonstrate our framework on a
practical application 
with an autonomous quadrotor helicopter learning a flight controller in different scenarios. Our method is tested on the Stanford-Berkeley Testbed of Autonomous Rotorcraft for Multi-Agent Control (STARMAC), using Ascending Technologies Pelican and Hummingbird quadrotors (Fig. \ref{fig:quad}). The system receives full state feedback from a VICON motion capture system. For the purpose of this series of experiments, the vehicle's dynamics are approximately decoupled through an on-board controller responsible for providing lateral stability around hover and vertical flight; our framework is then used to learn the feedback gains for a hybrid vertical flight controller. The learning and safety controllers were implemented and executed in MATLAB, on a Lenovo Thinkpad with an Intel~i7  processor
that communicated wirelessly with the vehicle's 1.99 GHz Quadcore Intel Atom processor.
This was all done using the Indigo version of the Robot Operating System (ROS) framework.
Reachability computations are executed using the Level Set Toolbox \cite{Mitchell2005a},
\new{%
employing the Lax-Friedrich approximation 
for the numerical Hamiltonian; 
a weighted essentially nonoscillatory scheme 
for spatial derivatives; 
and a third-order total variation diminishing Runge-Kutta scheme for the time derivative~%
\cite{Osher2003},~\cite{Shu1988}.
}
\new{%
Once the safety function and safety policy have been computed, they are stored as look-up tables that can be quickly consulted in constant time.%
}%

The purpose of the results presented here is not to advance the state of the art of quadrotor flight control or reinforcement learning techniques, but to illustrate how the proposed method can allow safe execution of an arbitrary learning-based controller without requiring any particular convergence rate guarantees. To fully demonstrate the reliability of our safe learning framework, in our first setup we let the vehicle begin its online learning in mid-air starting with a completely untrained controller.
The general functioning of the framework can be observed in the second flight experiment, in which the vehicle starts with a conservative model and iteratively computes empirical estimates of the disturbance, gradually expanding its computed safe set
while avoiding overreliance on poor predictions.
Finally, we include an experiment in which an unexpected disturbance is introduced into the system. The vehicle reacts by immediately applying the safe action dictated by its local safety policy and retracting from the perturbed region, successfully maintaining safety throughout its trajectory. We show how the absence of online guarantee validation in the same scenario can result in loss of safety.

We use an affine dynamical model of quadrotor vertical flight, with state equations:
\begin{equation}\label{eq:quad_dyn}
\begin{split}
\dot{x}_1 &= x_2\\
\dot{x}_2 &= k_T u + g + k_0 + d(x)\\
\end{split}
\end{equation}
where $x_1$ is the vehicle's altitude, $x_2$ is its vertical velocity, and $u\in[0,1]$ is the normalized motor thrust command. The gravitational acceleration is $g=-9.8$ m/s$^2$. The parameters of the affine model $k_T$ and $k_0$ are determined for the Pelican and the Hummingbird vehicles through a simple on-the-ground experimental procedure---a scale is used to measure the normal force reduction for different values of $u$.
The state constraint ${\K=\{x: \text{0 m} \le x_1\le \text{2.8 m}\}}$ encodes the position of the floor and the ceiling, which must be avoided.
Finally, $d$ is an unknown, state-dependent scalar disturbance term representing unmodeled forces in the system.
\new{
We learn $d(x)$ using a Gaussian process model, and generate $\Dx$ as the marginal 95\% confidence interval at each $x$.
We implement \emph{local} Bayesian guarantee validation,
conservatively approximating \eqref{eq:erf} by assuming
${s^j(x):=\sqrt{\text{var}(d^j(x) | X)}\approx \sqrt{\text{var}(d^j(x) | X_\text{old})}}$,
that is, neglecting the (favorable but often small) reduction in uncertainty due to $X_\text{new}$.
This was done for ease of prototyping.%
}%

As the learning-based controller, we choose an easily implementable policy gradient reinforcement learning algorithm~\cite{Kolter2009}, which learns the weights for a linear mapping from state features to control commands. Following~\cite{Gillula2012a}, we define different features for positive and negative velocities and position errors, since the (unmodeled) rotor dynamics may be different in ascending and descending flight. This can be seen as the policy gradient algorithm learning the feedback gains for a hybrid proportional-integral-derivative (PID) controller.

\subsection{From Fall to Flight}

To demonstrate the strength of Hamilton-Jacobi-based guarantees for safely performing learning-based control on a physical system, we first require a Pelican quadrotor to learn an effective vertical trajectory tracking controller with an arbitrarily poor initialization. To do this, the policy gradient algorithm is initialized with all feature weights set to $0$. The pre-computed safety controller (numerically obtained using \cite{Mitchell2005a}) is based on a conservative uncertainty bound of $\pm1.5$~m/s$^2$ everywhere in the state space; no new bounds are learned during this experiment. The reference trajectory requires the quadrotor to aggressively alternate between hovering at two altitudes, one of them ($1.5$~m) near the center of the room, the other ($0.1$~m) close to the floor.

This first experiment illustrates the interplay between the \emph{learning controller} and the \emph{safety policy}. The \emph{iterative safety re-computation} and \emph{Bayesian guarantee validation} components of the framework are not active here.
Consistently, this demonstration uses \eqref{eq:least-restrictive} as the least-restrictive safe policy. 

The experiment, shown in Fig. \ref{fig:learn2fly}, is initialized with the vehicle in mid-air. Since all feature weights are initially set to zero, the vehicle's initial action is to enter free fall. However, as the quadrotor is accelerated by gravity towards the floor, the boundary of the computed safe set is reached, triggering the intervention of the safety controller, which automatically overrides the learning controller and commands the maximum available thrust to the motors ($u=1$), causing the vehicle to decelerate and hover at a small distance from the ground. For the next few seconds, there is some chattering near the boundary of the safe set, and the policy gradient algorithm has some occasions to attempt to control the vehicle when it is momentarily pushed into the interior of the safe set. Initially it has little success, which leads the safety controller to continually intervene to prevent the quadrotor from colliding with the floor; this has the undesirable effect of slowing down the learning process,
since observations under this interference are uninformative about the behavior of the vehicle when actually executing the commands produced by the learning controller (which is an ``on-policy" algorithm).
However, at approximately $t=40$~s, the learning controller is able to make the vehicle ascend towards its tracking reference, retaining control of the vehicle for a longer span of time and accelerating the learning process. By $t=60$~s, the quadrotor is approximately tracking the reference, with the safety controller only intervening during the aggressive descent phase of the repeated trajectory, to ensure (under its conservative model) that there is no risk of a ground collision. The controller continues to learn in subsequent iterations, overall improving its tracking accuracy.

The remarkable result in this experiment is not in the quality of the learned tracking controller after only a few seconds of active exploration (a merit that corresponds to the reinforcement learning method \cite{Kolter2009}), but the system's ability to achieve competent performance at its task from an extremely poor initial policy while remaining safe at all times.

\begin{figure}
\begin{center}
\includegraphics[trim= 0cm 9cm 0cm 5cm, clip, width=.48\textwidth]{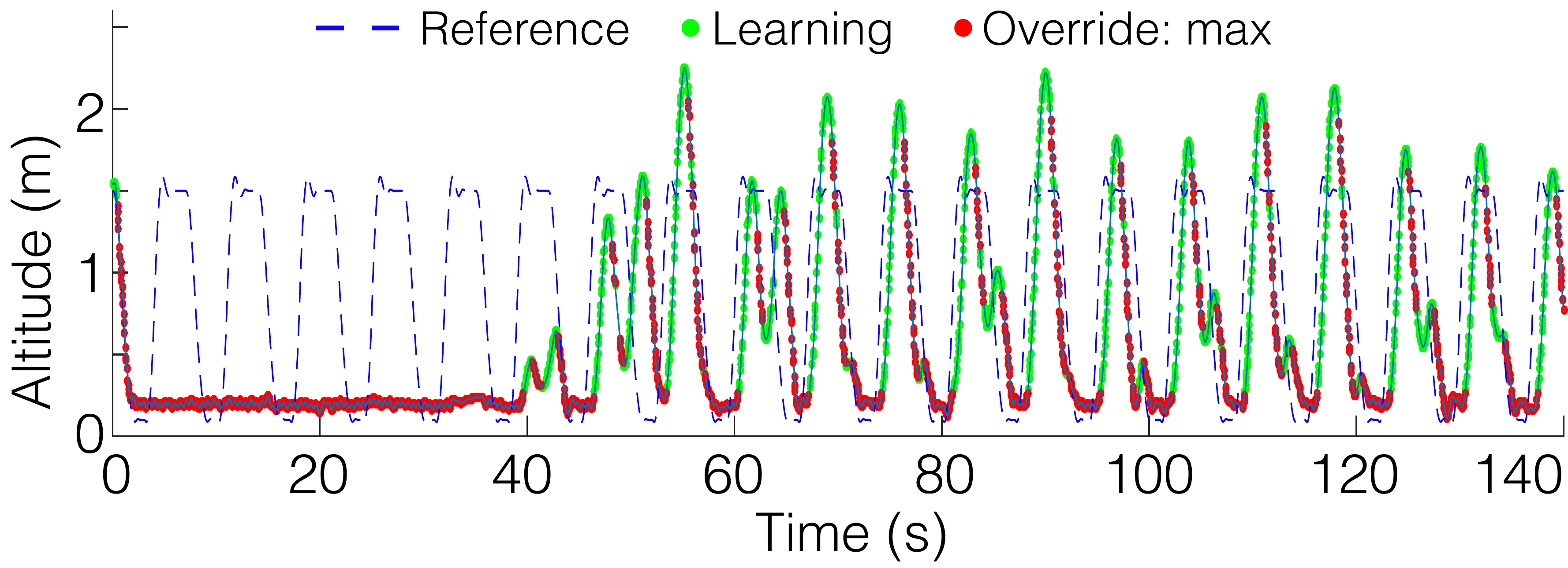}
\caption{Vehicle altitude and reference trajectory over time. Initial feedback gains are set to zero. When the learning controller (green) lets the vehicle drop, the safety control (red) takes over preventing a collision. Within a few seconds, the learned feedback gains allow rough trajectory tracking and are subsequently tuned as the vehicle attempts to minimize error.  \label{fig:learn2fly}}
\end{center}
\end{figure}

\subsection{When in Doubt}

In the second experiment, we demonstrate the iterative updating of the safe set and safety policy using observations of the system dynamics gathered over time, as well as the online validation of the resulting guarantees.
All components of the framework are active during the test, namely \emph{learning controller}, \emph{safety policy}, \emph{iterative safety re-computation}, and \emph{Bayesian guarantee validation}, with the main focus being on the latter two.

Here, the Pelican quadrotor attempts to safely track the same reference trajectory, while using the gathered information about the system's evolution to refine its notion of safety. 
In this case, the policy gradient learning algorithm is initialized to a hand-tuned set of parameter values.
The initial dynamic model available to the safety algorithm is identical to the one used in the previous experiment, with a uniform uncertainty bound of $\pm1.5$m/s$^2$. However, the system is now allowed to update this bound, throughout the state space, based on the disturbance posterior computed by a Gaussian process model.

To learn the disturbance function,
the system starts with a Gaussian process prior over $d(\cdot)$ defined by a zero mean function and a squared exponential covariance function:
\begin{equation}\label{eq:sqexp}
k(x,x')=\sigma_f^2 \exp \left(\frac{(x-x')^T\L^{-1}(x-x')}{2}\right),
\end{equation}
\noindent where $\L$ is a diagonal matrix, with $\L_i$ as the $i$th diagonal element, and $\theta_p=\begin{bmatrix}\sigma_f^2,\sigma_n^2,\L_{1},\L_{2} \end{bmatrix}$ are the hyperparameters, $\sigma_f^2$ being the signal variance, $\sigma_n^2$ the measurement noise variance, and the $\L_i$ the squared exponential's characteristic length for position and velocity respectively. The hyperparameters are chosen to maximize the marginal likelihood of the training data set, and are recomputed for each new batch of data when a new disturbance model $\hat\D(x)$ is generated for safety analysis.
Finally, the chosen prior mean and covariance kernel classes are both Lipschitz continuous, ensuring that all required technical conditions for the theoretical results hold (proofs are presented in the Appendix).

The expressions \eqref{eq:posterior}, \eqref{eq:single} give the marginal Gaussian process posterior on $d(x^*)$ for a query point $x^*$. To numerically compute the safe set, the system first evaluates \eqref{eq:Dx} to obtain the disturbance bound $\hat{\D}(\bf{x_i})$ at every point $\bf{x_i}$ on a state-space grid, as the $95\%$
confidence interval ($p=0.95$) of the Gaussian process posterior over $d(x)$; next, it performs the robust safety analysis by numerically solving the HJI equation \eqref{eq:HJI} on this grid (using \cite{Mitchell2005a}) and obtaining the safety function $V(x)$.

The trajectory followed by the quadrotor in this experiment is shown in Fig. \ref{fig:go_stop_go_traj}. The vehicle starts off with an \emph{a priori} conservative global bound on $d(x)$ and computes an initial conservative safe set $\Omega_1$ (Fig. \ref{fig:go_stop_go_sets}). It then attempts to track the reference trajectory avoiding the unsafe regions by transitioning to the safe control $u^*(x)$ on $\partial\Omega_1$.
The disturbance is measured and monitored online during this test,
under the local safety confidence criterion,
and found to be locally consistent with the initial conservative bound.
After collecting $10$~s of data, a new disturbance bound $\hat{\D}(x)$ is constructed using the corresponding Gaussian process posterior, from which
a second safety function $V_2(x)$ and safe set $\Omega_2$ are computed via Hamilton-Jacobi reachability analysis.
This process takes roughly $2$~seconds, and at approximately $t=12$~s the new safety guarantees and policy are substituted in.

The Pelican continues its flight under the results of this new safety analysis: however,
shortly after,
the vehicle measures values of $d$ that consistently approach the boundary of $\hat{\D}(x)$,
and reacts by
applying the safe control policy and locally climbing the computed safety function.
This confidence-based intervention
takes place several times during the test run, as the vehicle measures disturbances that lower its confidence in the local model bounds, effectively preventing the vehicle from approaching the ground.

After a few seconds, a new Gaussian process posterior is computed based on the first $20$~s of flight data, resulting in an estimated safe set $\Omega_3$, an intermediate result between the initial conservative $\Omega_1$ and the overly permissive $\Omega_2$ (Fig. \ref{fig:go_stop_go_sets}). The learning algorithm is then allowed to resume tracking under this new safety analysis, and no further safety overrides take place due to loss of safety confidence.

This experiment demonstrates the algorithm's ability to safely refine its notion of safety
as more data become available,
without requiring the process to consist in a
series 
of strictly conservative under-approximations.

\begin{figure}
\begin{center}
\includegraphics[trim= 0.75cm 0cm 1cm 0cm, clip, width=.48\textwidth]{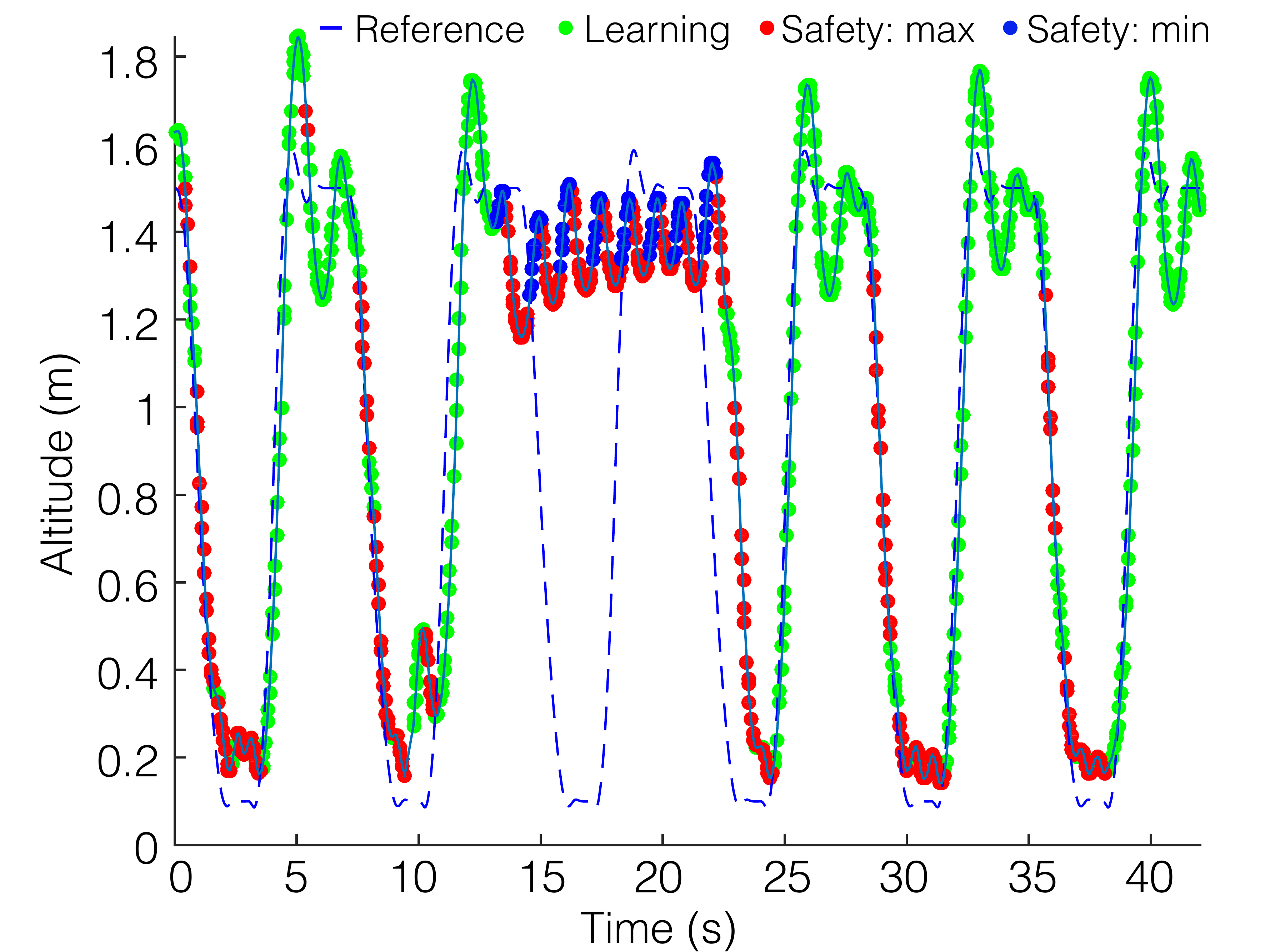}
\caption{Vehicle altitude and reference trajectory over time. After flying with an initial conservative model, the vehicle computes a first Gaussian process model of the disturbance with only a few data points, resulting in an insufficiently accurate bound. The safety policy detects the low confidence and refuses to follow the reference to low altitudes. Once a more accurate disturbance bound is computed, tracking is resumed, with a less restrictive safe set than the original one. \label{fig:go_stop_go_traj}}
\end{center}
\end{figure}

\begin{figure}
\begin{center}
\includegraphics[trim= 1.25cm 0cm 0.5cm 0cm, clip, width=.48\textwidth]{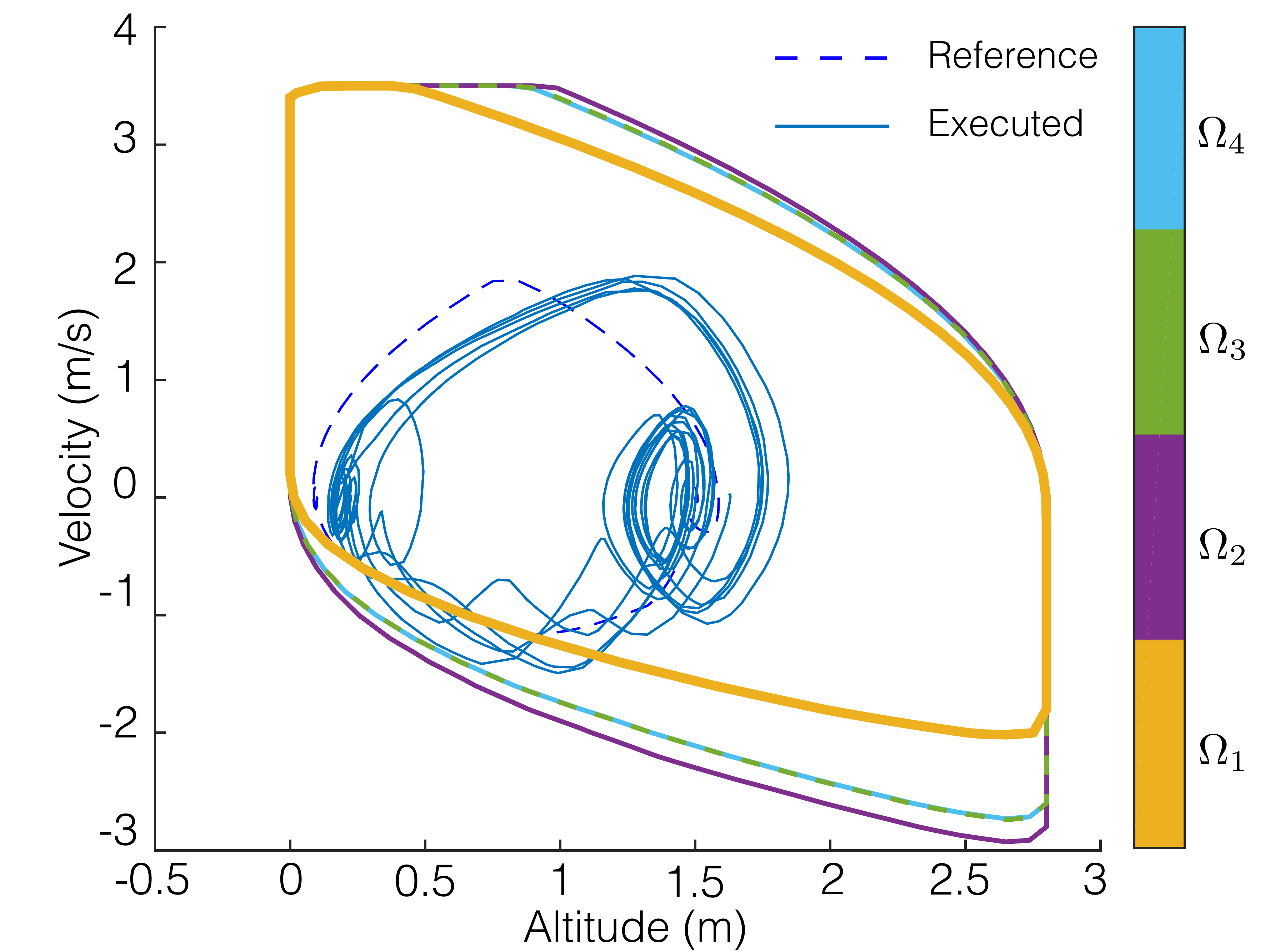}
\caption{Safe sets computed online by the safety algorithm as it gathers data and successively updates its Gaussian process disturbance model.
The vehicle's trajectory eventually leaves the initial, conservative $\Omega_1$, but remains in the converged safe set ($\Omega_4$) at all times, even \emph{before} this set is computed.
While the intermediate set $\Omega_2$ would have been overly permissive, this is remedied by the 
intervention of the safety controller as soon as the model is observed to behave poorly.\label{fig:go_stop_go_sets}}
\end{center}
\end{figure}

\subsection{Gone with the Wind}

In this last experimental result, we display the efficacy of online safety guarantee validation
in handling alterations in operating conditions unforeseen by the system designer. All components of the framework are active, except for the \emph{iterative safety re-computation}, which is not used in this case.

This experiment is performed using the lighter Hummingbird quadrotor, which is more agile than the Pelican but also more susceptible to wind.
We initialize the disturbance set to a conservative range of $\pm 2$~m/s$^2$, which amply captures the error in the double-integrator model for vertical flight.
The vehicle tracks a slow sinusoidal trajectory using policy gradient \cite{Kolter2009} to improve upon the manually initialized controller parameters. At approximately $t=45$~s an unmodeled disturbance is introduced by activating a fan aimed laterally at the quadrotor. The fan is positioned on the ground and angled slightly upward, so that its effect increases as the quadrotor flies closer to the ground. The presence of the airflow causes the attitude and lateral position controllers to use additional control authority to stabilize the quadrotor, which couples into the vertical dynamics as an unmodeled force.

The experiment is performed with and without the \emph{Bayesian guarantee validation} component, with resulting trajectories shown in Fig. \ref{fig:fan_time}. Without validation, the quadrotor violates the constraints, repeatedly striking the ground. With validation, the fan's airflow is quickly detected as a discrepancy with the model near the floor,
and the safety controller override is triggered.
The vehicle avoids entering the affected region for the remainder of the flight. Although only the local confidence method is used, providing a strictly local safety guarantee, the safe controller succeeds in maintaining safety throughout the experiment.
This provides strong evidence suggesting that, beyond its theoretical guarantees, the local Bayesian analysis also constitutes an effective best-effort approach to safety in more general conditions, given limited computational resources and available knowledge about the system.

\begin{figure}
\begin{center}
\includegraphics[trim= 1cm 0cm 1cm 0cm, clip, width=.49\textwidth]{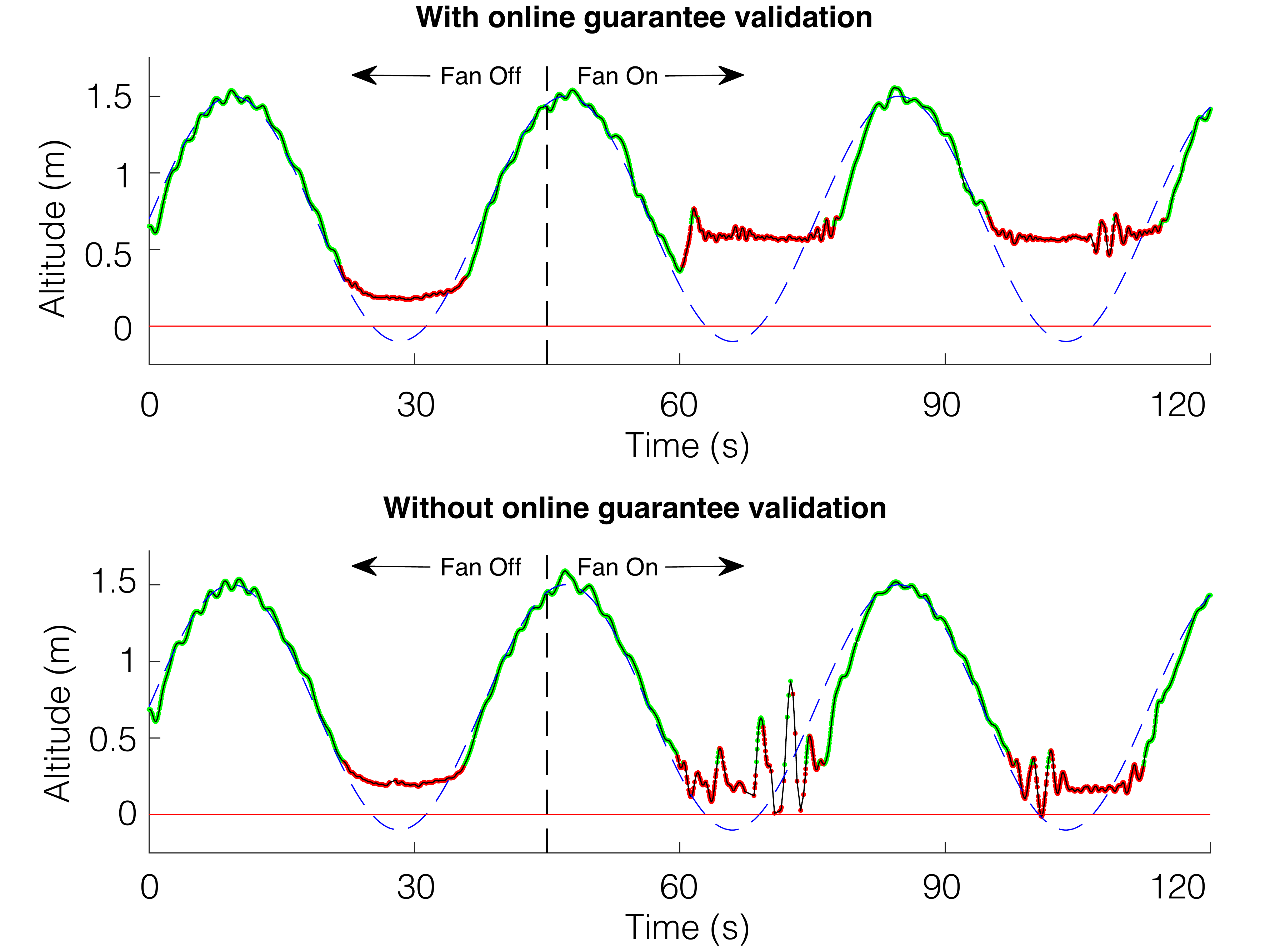}
\caption{Vehicle altitude and reference trajectory over time, shown with and without online model validation. After the fan is turned on, the vehicle checking local model reliability detects the inconsistency and overrides the learning controller, avoiding the region with unmodeled dynamics; the vehicle without model validation enters this region and collides with the ground multiple times. The behavior is repeated when the reference trajectory enters the perturbed region a second time. \label{fig:fan_time}}
\end{center}
\end{figure}

\section{Conclusion}

We have introduced a safe learning framework that combines robust reachability guarantees from control theory with Bayesian analysis based on empirical observations. This results in a minimally restrictive supervisory controller that can allow an arbitrary learning algorithm to safely explore its state and strategy spaces. As more data are gathered online, the framework allows the system to probabilistically reason about the validity of its robust model-based safety guarantees in light of the latest empirical evidence.

We firmly believe that providing strong and practically useful safety guarantees for systems that navigate unstructured environments requires a rapprochement between model-based and data-driven techniques, often regarded as a dichotomy by both theoreticians and practitioners. With this work we intend to provide mathematical arguments and empirical evidence of the potential that the two approaches hold
when used in conjunction.

Scaling up safety certificates as intelligent systems achieve increasing complexity poses an important open research problem. Our prediction is that, as autonomous systems interact increasingly closely with human beings, the graceful interplay of safety and learning, combining theoretical guarantees with empirical grounding, will become central to their success.
\newpage
\section*{Appendix}
Restricting one of the control inputs $ d$ to a \emph{state-dependent} bound $\Dx$ introduces questions as to whether a unique Carath\'eodory solution to \eqref{eq:xdot} continues to exist. The basic existence and uniqueness theorems \cite{Coddington1955} assume fixed control sets. If the variation in the control sets can instead be expressed through the dynamic equation itself without breaking the continuity conditions, then it is easy to extend the classical result to at least a class of space-dependent control sets. We introduce two technical assumptions, which give sufficient conditions for the existence and uniqueness of a solution to the dynamical equations, and prove that any disturbance set $\Dx$ obtained from a Gaussian process model with Lipschitz prior mean and covariance kernel satisfies these assumptions.

\begin{assumption}\label{ass:retract}
For all $x\in\RR^n$, $\hat\D(x)$ is a closed 
deformation retract of $\D$, that is, there exists a continuous map $H_x: \D\times[0,1]\to\hat\D(x)$ such that for every $ d\in\D$ and $\hat d\in\hat\D(x)$,
$H_x( d,0)= d, \, H_x( d,1)\in\hat\D(x), \, H_x(\hat d,1) = \hat d.$
\end{assumption}
\begin{assumption}\label{ass:r}
Let $r: \RR^n\times\D\to\D$ be such that ${r(x, d)=H_x( d,1)}$ as defined above. Then $r$ is Lipschitz continuous in $x$, and uniformly continuous in $ d$.
\end{assumption}

Intuitively, the first assumption means that $\D$ can be continuously deformed into $\Dx$ for any $x$, while the second prevents the disturbance bound $\hat\D(x)$ from changing abruptly as one moves in the state space.
The retraction map $r$ allows us to absorb the state-dependence of the disturbance bound into the system dynamics, enabling us to use the standard analysis for differential games, which considers measurable time signals drawn from fixed input sets. This is formalized in the following result.

\begin{proposition}
The saturated system dynamics $\tilde{f}_{\hat\D}(x,u, d):=f\big(x,u,r(x, d)\big)$ are bounded and uniformly continuous in all variables, and Lipschitz in $x$.

\begin{proof}
Boundedness and uniform continuity of $\tilde{f}_{\hat\D}$ in $u$ are trivially inherited from $f$; we therefore focus on $d$ and $x$.

First, since $r$ is uniformly continuous in $ d$, and $f$ is Lipschitz (hence uniformly continuous) in its third argument, we have by composition that $\tilde{f}_{\hat\D}$ is uniformly continuous in $ d$.

Lipschitz continuity in $x$ is less immediate due to its appearance in both the first and third arguments of $f$. Again by composition, Lipschitz continuity of $r$ in $x$ implies that $f\big(y,u,r(\cdot, d)\big)$ is also Lipschitz for all $y\in\RR^n$, $u\in\U$ and $ d\in\D$. Letting $L_r$ be the Lipschitz constant of $r$ and $L_x$ be the Lipschitz constant of $f$ in its first argument, we have, for any $x,\tilde x \in\RR^n$:
\[\begin{split} &|f(x,u,r(x, d)) - f(\tilde{x},u,r(\tilde{x}, d))| \\
&\le |f(x,u,r(x, d)) - f(\tilde{x},u,r(x, d))| + |f(\tilde{x},u,r(x, d)) - f(\tilde{x},u,r(\tilde{x}, d))| \\
&\le (L_x + L_ d L_r)|x - \tilde{x}| \end{split}\]
Thus $\tilde{f}_{\hat\D}(\cdot,u, d)=f(\cdot,u,r(\cdot, d))$ is indeed Lipschitz in $x$.
\end{proof}
\end{proposition}
\begin{corollary}
The dynamical system given by \eqref{eq:xdot} with $ d$ constrained to lie in a state-dependent set $\Dx$ satisfying Assumptions \ref{ass:retract} and \ref{ass:r} has a unique continuous solution in the extended (Carath\'eodory) sense.
\end{corollary}

The above result is important in that it guarantees existence and uniqueness of system trajectories for any state-dependent disturbance bound $\hat \D(\cdot)$ that satisfies Assumptions \ref{ass:retract} and \ref{ass:r}. Moreover, the above construction allows us to transform the system dynamics $f(x,u, d)$ with $ d\in\hat\D(x)$ into the standard form with fixed input sets (i.e. $\tilde{f}_{\hat\D}(x,u, d)$ with $u\in\U$, $ d\in\D$), so that all results from the differential games literature can readily be applied to our formulation.

Let us now see that the posterior mean and standard deviation of the components of $d(x)$ are Lipschitz continuous functions of the state $x$ under our Gaussian process framework.

\begin{proposition}\label{prop:Lipschitz}
Let the prior mean function $\mu^j$ be Lipschitz continuous, and the covariance kernel $k^j$ be jointly Lipschitz continuous,
for the $j$th component of the disturbance function $d(x)$. Then the posterior mean $\bar d^j(x)$ and standard deviation $\sigma^j(x)$, as given by \eqref{eq:posterior}, \eqref{eq:single} are Lipschitz continuous in $x$.

\begin{proof}
The result follows from applying the hypotheses to \eqref{eq:posterior}, \eqref{eq:single}. Note that the standard deviation $\sigma^j(x)$ is the square root of the variance in \eqref{eq:single}; the square root function is Lipschitz everywhere except at 0, and Bayesian inference under nondegenerate prior and likelihood never results in 0 posterior variance. Thus $\sigma^j(\cdot)$ is also Lipschitz. 
\end{proof}
\end{proposition}

The following proposition relates the state-dependent bound $\Dx$ obtained from Gaussian process regression to Assumptions \ref{ass:retract} and \ref{ass:r}, ensuring that the dynamical system \eqref{eq:xdot} is well-defined, and therefore the associated dynamic game be solved using the methods presented in Section \ref{sec:analysis}.

\begin{proposition}\label{prop:assumptions}
Let the prior mean function $\mu^j$ be Lipschitz continuous, and the covariance kernel $k^j$ be jointly Lipschitz continuous in its two variables, for all components $j$ of the disturbance function $d(x)$. Then the disturbance bound $\Dx$, as defined in \eqref{eq:Dx}, satisfies Assumptions \ref{ass:retract} and \ref{ass:r}.

\begin{proof}
Assumption \ref{ass:retract} holds independently of the Lipschitz condition. The bound $\Dx$ given by \eqref{eq:Dx} is a compact convex set in $\D$. As a result, the retraction map $\pi_x:\D\to\Dx$ assigning every $ d\in\D$ its (unique) nearest point in $\Dx$ is a Lipschitz continuous function (with Lipschitz constant equal to 1); of course with $\pi_x(\hat d)=\hat d$ for all $\hat d\in\Dx$. Then, the function $H_x( d,t):= (1-t)  d + t \pi_x( d)$ is continuous by composition and further satisfies $H_x( d,0)= d, H_x( d,1)\in\hat\D(x), H_x(\hat d,1) = \hat d$ for all $ d\in\D$ and $\hat d\in\Dx$.

Assumption \ref{ass:r} can be shown to hold by noting that the extrema of each of the intervals in \eqref{eq:Dx} are affine in $\bar d^j(x)$ and $\sigma^j(x)$, which are Lipschitz continuous in $x$ by Proposition~\ref{prop:Lipschitz}. This implies that the position of all vertices of the hyperrectangle in \eqref{eq:Dx} varies as a Lipschitz continuous function of $x$, and so does, as a result, the nearest point in $\Dx$ to any fixed $ d\in\D$. The map $r(x, d)=\pi_x( d)$ is hence Lipschitz continuous in $x$. Finally, since $\pi_x$ is Lipschitz continuous in $ d$, $r$ is also uniformly continuous in $d$.
\end{proof}
\end{proposition}

Finally, we can show that, under the same Lipschitz assumptions on the Gaussian process prior, the disturbance bound $\Dx$ is Lipschitz continuous under the Hausdorff distance, which we required in Proposition \ref{prop:locally_correct}.

\begin{proposition}
Let the prior mean function $\mu^j$ be Lipschitz continuous, and the covariance kernel $k^j$ be jointly Lipschitz continuous in its two variables, for all components $j$ of the disturbance function $d(x)$. Then the set-valued map $\hat\D$ is Lipschitz continuous under the Hausdorff distance.

\begin{proof}
Since the disturbance set $\Dx$ given by \eqref{eq:Dx} is a hyperrectangle,
the Hausdorff distance between
the disturbance sets
$\hat\D(x_1)$ and $\hat\D(x_2)$ is upper-bounded by the maximum distance between any pair of corners:
\[
\delta(x_1,x_2):=d_H\big(\hat\D(x_1),\hat\D(x_2)\big) \le \max_{i} \max_{k} |c_i-c_k|
\enspace ,
\]
with $c_i$,$c_k$ used to enumerate all corners of each of the two hyperrectangles. For simplicity of exposition, we use the equivalence of all norms in $\RR^{d_n}$ to upper-bound the above, arbitrary norm in $\RR^{n_d}$, by the infinity norm, up to a constant factor $m$, which in combination with \eqref{eq:Dx} gives:
\[
\delta(x_1,x_2) \le m \cdot \max_{j} \big( |\bar{d^j}(x_1 ) - \bar{d^j}(x_2 )| + |z\sigma^j(x_1)-z\sigma^j(x_2)| \big)
\, .
\]
Now, by Proposition \ref{prop:Lipschitz}, $\bar{d}^j(x)$ and $\sigma^j(x)$ are Lipschitz continuous in $x$; let their respective constants be $L^j_\mu$ and $L^j_\sigma$. We then have
\[
d_H\big(\hat\D(x_1),\hat\D(x_2)\big) \le m \cdot\max_{j}  \big(L^j_\mu + z L^j_\sigma\big) |x_1 - x_2 |
\enspace ,
\]
which proves Hausdorff Lipschitz continuity of the set-valued map $\hat\D$, with a Lipschitz constant $L_{\hat\D}$ upper bounded by ${m\cdot\max_{j}  L^j_\mu + z L^j_\sigma}$.
\end{proof}
\end{proposition}

\printbibliography
\newpage

\begin{IEEEbiography}[{\includegraphics[width=1in,height=1.25in,clip,keepaspectratio]{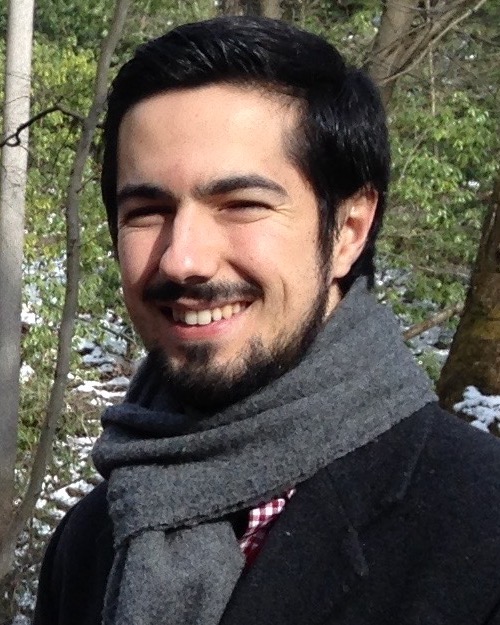}}]{Jaime F. Fisac} is a Ph.D. candidate in Electrical Engineering and Computer Sciences at the University of California, Berkeley. He received a B.S./M.S. degree in Electrical Engineering from the Universidad Polit{\'e}cnica de Madrid, Spain, in 2012, and a M.Sc. in Autonomous Vehicle Dynamics and Control from Cranfield University, UK, in 2013. He is a recipient of the ``la~Caixa'' Foundation Fellowship (2013-2015). His research interests lie in control theory, artificial intelligence, and cognitive science, with a focus on safety for robotic and AI systems operating closely with people.
\end{IEEEbiography}
\vspace{-1.32cm}
\begin{IEEEbiography}[{\includegraphics[width=1in,height=1.25in,clip,keepaspectratio]{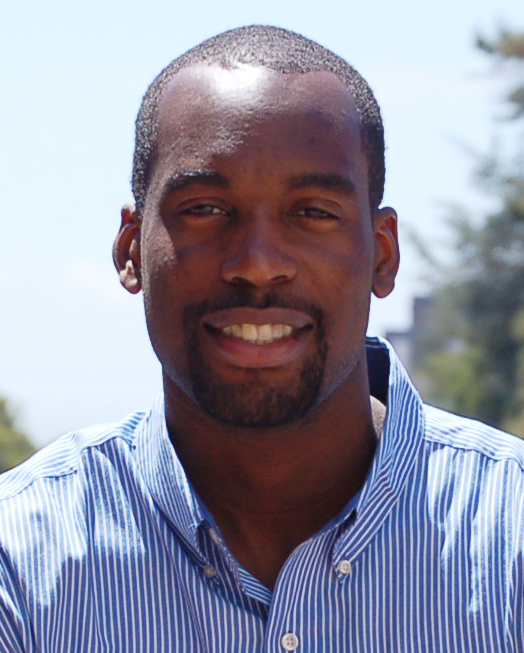}}]{Anayo K. Akametalu} is a PhD. candidate in Electrical Engineering and Computer Sciences at the University of California, Berkeley. He obtained his B.S. degree in Electrical Engineering from the University of California, Santa Barbara in 2012. His research interests lie at the intersection of control theory and reinforcement learning. He has been funded through the National Science Foundation Bridge to Doctorate Fellowship, UC Berkeley Chancellor's Fellowship, and GEM Fellowship. In 2016 he took a break from his graduate studies to work at Apple. \end{IEEEbiography}
\vspace{-1.32cm}
\begin{IEEEbiography}[{\includegraphics[width=1in,height=1.25in,clip,keepaspectratio]{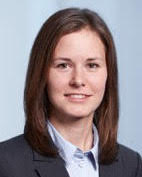}}]{Melanie N. Zeilinger} is an Assistant Professor at the Department of Mechanical and Process Engineering, ETH Zurich, Switzerland. She received a Diploma in Engineering Cybernetics (2006) from the University of Stuttgart, Germany, and a Ph.D. in Electrical Engineering (2011) from ETH Zurich. She was a Marie Curie Fellow and Postdoctoral Researcher at the Max Planck Institute for Intelligent Systems, Germany (until 2015), UC Berkeley, CA (2012-14), and EPFL, Switzerland (2011-12). Her research interests include distributed control and optimization, as well as safe learning-based control, with applications to human-in-the-loop systems.
\end{IEEEbiography}
\vspace{-1.32cm}
\begin{IEEEbiography}[{\includegraphics[width=1in,height=1.25in,clip,keepaspectratio]{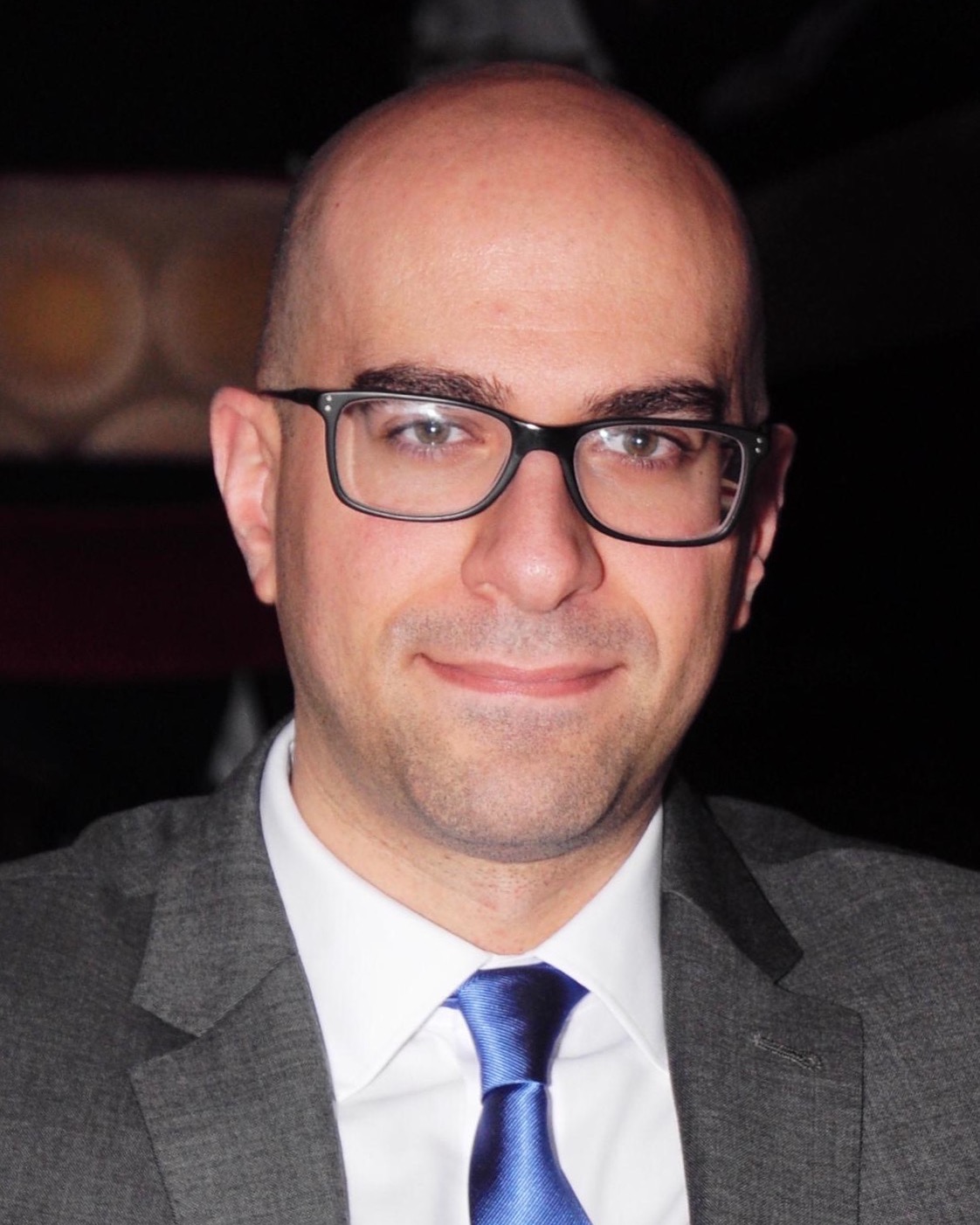}}]{Shahab Kaynama} is a Senior Algorithms Architect at Apple. Between 2014-2017 he was with Clearpath Robotics as an Autonomy Team Lead for Navigation and Controls, where he helped launch its two OTTO flagship products. He was a Postdoctoral Research Scholar at UC Berkeley (since 2012) and at the University of British Columbia (since 2014). Dr. Kaynama received a Ph.D. (2012) in Electrical and Computer Engineering from the University of British Columbia, Canada and an M.Sc. (2006) in Advanced Control and Systems Engineering from the University of Manchester, UK.
\end{IEEEbiography}
\vspace{-1.32cm}
\begin{IEEEbiography}[{\includegraphics[width=1in,height=1.25in,clip,keepaspectratio]{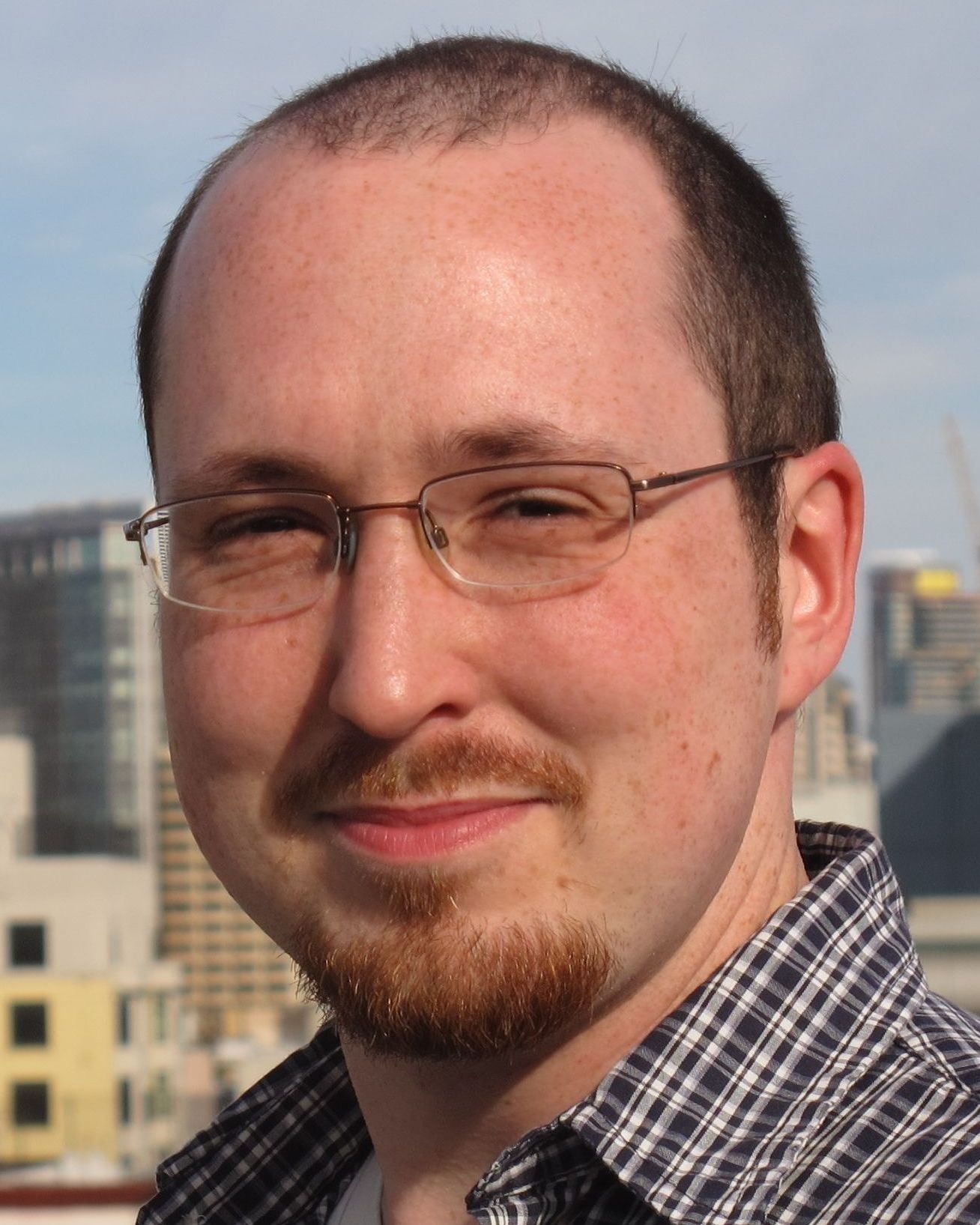}}]{Jeremy Gillula} received a B.S. in Computer Science from Caltech, CA, in 2006, and a M.S. and Ph.D. in Computer Science from Stanford University, CA, in 2011 and 2013. After spending eight months as a Postdoctoral Researcher at UC Berkeley, he became a Staff Technologist at the Electronic Frontier Foundation, a San Francisco-based nonprofit organization dedicated to defending civil liberties in the digital world.
\end{IEEEbiography}
\vspace{-1.32cm}
\begin{IEEEbiography}[{\includegraphics[width=1in,height=1.25in,clip,keepaspectratio]{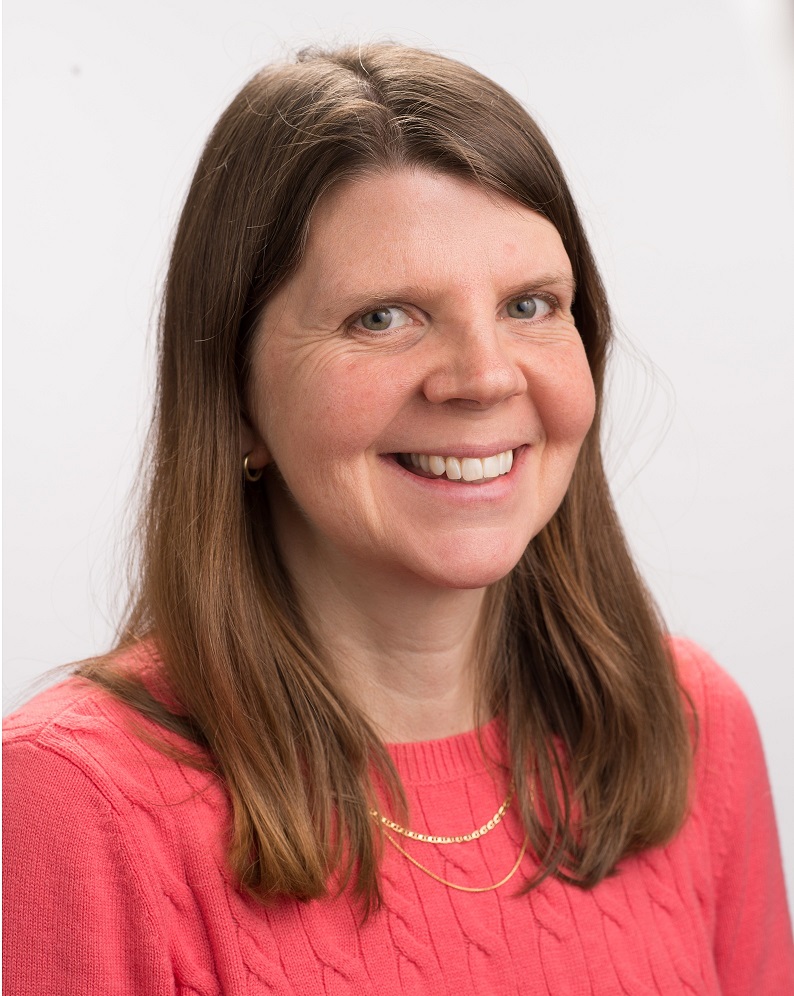}}]{Claire J. Tomlin} is the Charles A. Desoer Professor of Engineering in Electrical Engineering and Computer Sciences at the University of California, Berkeley. She was an Assistant, Associate, and Full Professor in Aeronautics and Astronautics at Stanford from 1998 to 2007, and in 2005 joined Berkeley. Claire works in the area of control theory and hybrid systems, with applications to air traffic management, UAV systems, energy, robotics, and systems biology. She is a MacArthur Foundation Fellow (2006) and an IEEE Fellow (2010), and in 2010 held the Tage Erlander Professorship of the Swedish Research Council at KTH in Stockholm.
\end{IEEEbiography}

\end{document}